\documentclass[letterpaper]{article}
\usepackage{aaai25} 
\usepackage{times} 
\usepackage{helvet} 
\usepackage{courier} 
\usepackage[hyphens]{url} 
\usepackage{graphicx} 
\usepackage{natbib} 
\usepackage{caption} 
\usepackage{algorithm}
\usepackage{algorithmic}

\usepackage{subcaption}
\usepackage{amsmath}
\allowdisplaybreaks
\usepackage{amsthm}
\usepackage{amsfonts}
\usepackage{url}
\usepackage{booktabs}
\usepackage{multirow}
\newcommand{\cD}{\mathcal{D}}
\newcommand{\bA}{\mathbf{A}}
\newcommand{\bB}{\mathbf{B}}
\newcommand{\trrm}{\mathrm{tr}}
\newcommand{\term}{\mathrm{te}}

\newcommand{\auxrm}{\mathrm{aux}}

\newcommand{\mrm}{\mathrm{m}}
\newcommand{\nmrm}{\mathrm{nm}}
\newcommand{\adarm}{\mathrm{ada}}
\newcommand{\cL}{\mathcal{L}}
\newcommand{\plrm}{\mathrm{PL}}
\newcommand{\splrm}{\mathrm{SPL}}
\newtheorem{lemma}{Lemma}
\newtheorem{proposition}{Proposition}
\newtheorem{definition}{Definition}

\usepackage{newfloat}
\usepackage{listings}
\DeclareCaptionStyle{ruled}{labelfont=normalfont,labelsep=colon,strut=off} 
\floatstyle{ruled}
\newfloat{listing}{tb}{lst}{}
\floatname{listing}{Listing}
\title{Privacy-Preserving Low-Rank Adaptation against Membership\\Inference Attacks for Latent Diffusion Models}
\author{
    Zihao Luo\textsuperscript{\rm 1}\equalcontrib,
    Xilie Xu\textsuperscript{\rm 2}\equalcontrib,
    Feng Liu\textsuperscript{\rm 3},
    Yun Sing Koh\textsuperscript{\rm 1},\\
    Di Wang\textsuperscript{\rm 4},
    Jingfeng Zhang\textsuperscript{\rm 1 \rm 4}\thanks{Correspondence to: Jingfeng Zhang \textless\protect\url{jingfeng.zhang@auckland.ac.nz}\textgreater.}
}
\affiliations{
    \textsuperscript{\rm 1}The University of Auckland\\
    \textsuperscript{\rm 2}The National University of Singapore\\
    \textsuperscript{\rm 3}The University of Melbourne\\
    \textsuperscript{\rm 4}King Abdullah University of Science and Technology\\

}

\usepackage{bibentry}

\begin{document}

\maketitle

\begin{abstract}
Low-rank adaptation (LoRA) is an efficient strategy for adapting latent diffusion models (LDMs) on a private dataset to generate specific images by minimizing the adaptation loss. However, the LoRA-adapted LDMs are vulnerable to membership inference (MI) attacks that can judge whether a particular data point belongs to the private dataset, thus leading to the privacy leakage. To defend against MI attacks, we first propose a straightforward solution: \underline{M}embership-\underline{P}rivacy-preserving \underline{LoRA} (MP-LoRA). MP-LoRA is formulated as a min-max optimization problem where a proxy attack model is trained by maximizing its MI gain while the LDM is adapted by minimizing the sum of the adaptation loss and the MI gain of the proxy attack model. However, we empirically find that MP-LoRA has the issue of unstable optimization, and theoretically analyze that the potential reason is the unconstrained local smoothness, which impedes the privacy-preserving adaptation. To mitigate this issue, we further propose a \underline{S}table \underline{M}embership-\underline{P}rivacy-preserving \underline{LoRA} (SMP-LoRA) that adapts the LDM by minimizing the ratio of the adaptation loss to the MI gain. Besides, we theoretically prove that the local smoothness of SMP-LoRA can be constrained by the gradient norm, leading to improved convergence. Our experimental results corroborate that SMP-LoRA can indeed defend against MI attacks and generate high-quality images. 
\end{abstract}

\begin{links}
    \link{Code}{https://github.com/WilliamLUO0/StablePrivateLoRA}
\end{links}

\section{Introduction}
Generative diffusion models~\citep{Ho2020denoising,Song2020score} are leading a revolution in AI-generated content, renowned for their unique generation process and fine-grained image synthesis capabilities. 
Notably, the Latent Diffusion Model (LDM)~\citep{Rombach2022high,Podell2023sdxl} stands out by executing the diffusion process in latent space, enhancing computational efficiency without compromising image quality.
Thus, LDMs can be efficiently adapted to generate previously unseen contents or styles~\citep{Meng2021sdedit,Gal2022image,Ruiz2023dreambooth,Zhang2023adding}, thereby catalyzing a surge across multiple fields, such as facial generation~\citep{Huang2023collaborative,Xu2024personalized} and medicine~\citep{Kazerouni2022diffusion,Shavlokhova2023finetuning}.

Among various adaptation methods, Low-Rank Adaptation (LoRA)~\citep{Hu2021lora} is the superior strategy for adapting LDMs by significantly reducing computational resources while ensuring commendable performance with great flexibility.
Compared to the full fine-tuning method which fine-tunes all parameters, LoRA optimizes the much smaller low-rank matrices, making the training more efficient and lowering the hardware requirements for adapting LDMs~\citep{Hu2021lora}.
By performing the low-rank decomposition of the transformer structure within the LDM, LoRA offers performance comparable to fine-tuning all LDM parameters~\citep{Pedro2023lorasd}.
Moreover, LoRA allows flexible sharing of a pre-trained LDM to build numerous small LoRA modules for various tasks.

However, recent studies~\citep{Pang2023black,Pang2023white,Dubinski2024towards} have pointed out that adapted LDMs are facing the severe risk of privacy leakage. 
The leakage primarily manifests in the vulnerability to Membership Inference (MI) attacks~\citep{Shokri2017membership}, which utilize the model's loss of a data point to differentiate whether it is a member of the training dataset or not.
As shown in Figure~\ref{fig:fig1d_Performance}, the LoRA-adapted LDM (red circle marker) exhibits an incredibly high Attack Success Rate (ASR) of 82.27\%.

\begin{figure*}[t]
    \centering
    \begin{minipage}{0.68\textwidth}
        \begin{subfigure}[t]{0.49\linewidth}
            \centering
            \includegraphics[width=0.96\linewidth,height=9cm,keepaspectratio]{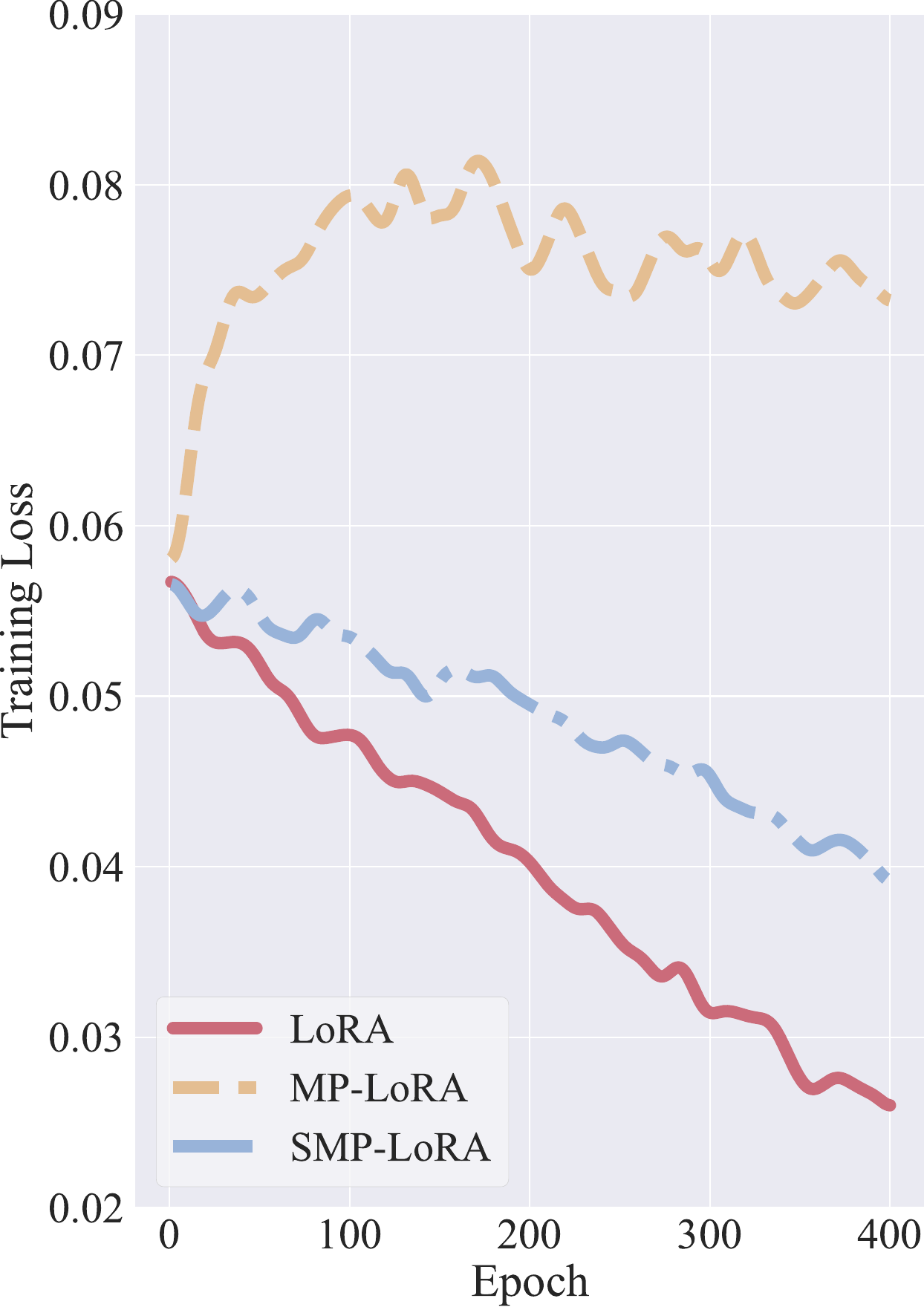}
            \caption{Training loss}
            \label{fig:fig1a_TrainingLoss}
        \end{subfigure}%
        \hfill
        \begin{subfigure}[t]{0.49\linewidth}
            \centering
            \includegraphics[width=\linewidth,height=9cm,keepaspectratio]{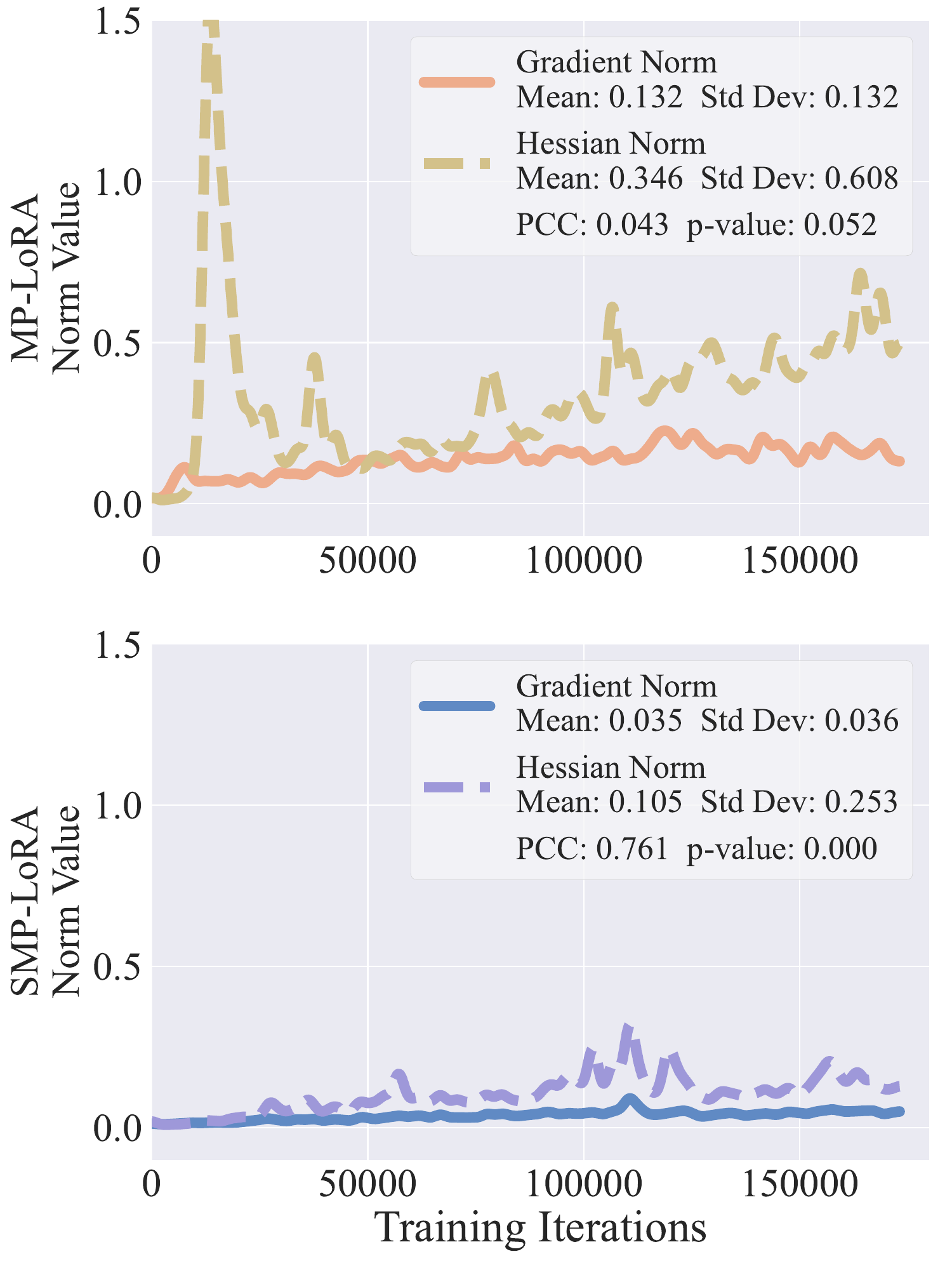}
            \caption{Gradient scale}
            \label{fig:fig1b_GradientHessian}
        \end{subfigure}
    \end{minipage}%
    \begin{minipage}{0.28\textwidth}
        \begin{subfigure}[t]{\linewidth}
            \centering
            \includegraphics[width=0.7\linewidth,height=\linewidth,keepaspectratio]{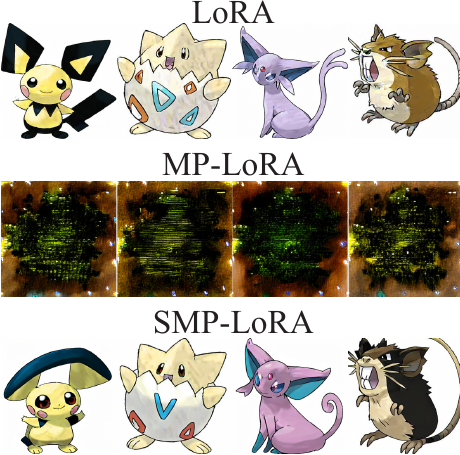}
            \caption{Generated results}
            \label{fig:fig1c_GeneratedResults}
        \end{subfigure}\\ 
        \begin{subfigure}[b]{\linewidth}
            \centering
         \includegraphics[width=\linewidth,height=\linewidth,keepaspectratio]{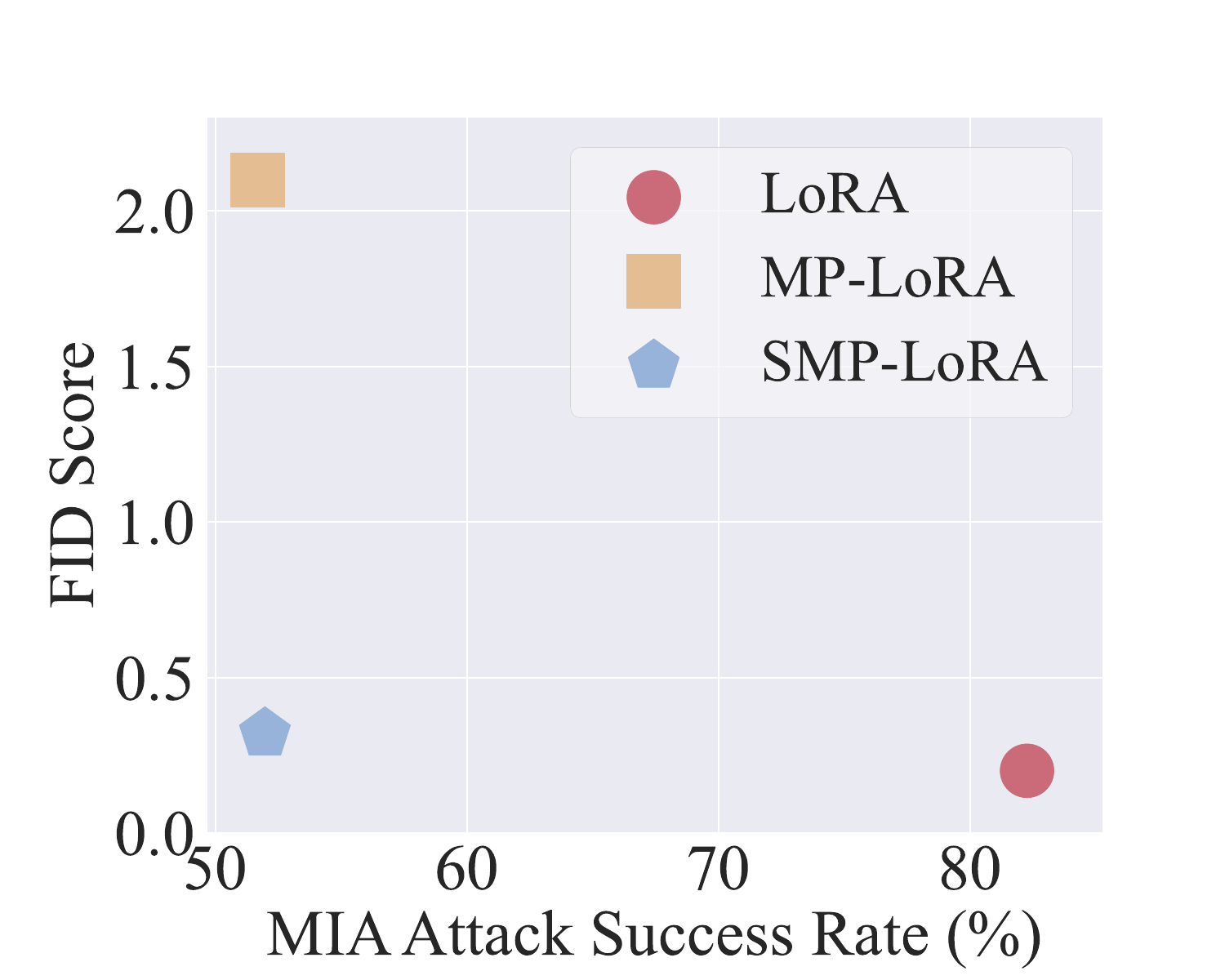} 
            \caption{Performance}
            \label{fig:fig1d_Performance}
        \end{subfigure}
    \end{minipage}
    \caption{Figure~\ref{fig:fig1a_TrainingLoss} shows the trajectory of the training loss during the adaptation process via LoRA, MP-LoRA, and SMP-LoRA on the Pokemon dataset. Figure~\ref{fig:fig1b_GradientHessian} displays the mean and standard deviation of the gradient norms and Hessian norms for MP-LoRA and SMP-LoRA throughout the training iterations. It also presents the Pearson correlation coefficients (PCC) and p-values assessing their correlation. Note that each epoch contains 433 training iterations. Figures~\ref{fig:fig1c_GeneratedResults} and~\ref{fig:fig1d_Performance} demonstrate the generated images and a comparison of evaluation metrics including FID Score and MI attack success rate (ASR). MP-LoRA preserves membership privacy but compromises image generation capability. In contrast, SMP-LoRA effectively preserves membership privacy while maintaining the quality of the generated image, demonstrating its effectiveness in defending against MI attacks without significant loss of functionality. Extensive generated images are visualized in Appendix~\ref{app:Visul_mpsmplora}.}
    \label{fig:fig1}
\end{figure*}

To mitigate the issue of privacy leakage, we make the first effort to propose a \underline{M}embership-\underline{P}rivacy-preserving \underline{LoRA} (MP-LoRA) method, which is formulated as a min-max optimization problem. 
Specifically, in the inner maximization step, a proxy attack model is trained to maximize its effectiveness in inferring membership privacy which is quantitatively referred to as MI gain.
In the outer minimization step, the LDM is adapted by minimizing the sum of the adaptation loss and the MI gain of the proxy attack model to enhance the preservation of membership privacy.

However, the vanilla MP-LoRA encounters an issue of effective optimization of the training loss, as evidenced in the orange dashed line of Figure~\ref{fig:fig1a_TrainingLoss}.
We theoretically find that during MP-LoRA, the local smoothness, quantified by the Hessian norm (the norm of the Hessian matrix)~\citep{Bubeck2015convex}, is independent of and not bounded by the gradient norm (see Proposition~\ref{pro:privatelora} for details).
This independence hinders the privacy-preserving adaptation of the MP-LoRA~\citep{Zhang2019gradient}, thus impeding optimizing the training loss.
Besides, we empirically show that the correlation between the Hessian norm and the gradient norm during MP-LoRA is insignificant.
This is manifested by the Pearson Correlation Coefficient (PCC) of 0.043 and the p-value above 0.05, as shown in the upper panel of Figure~\ref{fig:fig1b_GradientHessian}, which corroborates our theoretical analyses.

To stabilize the optimization procedure of MP-LoRA, we further propose a \underline{S}table \underline{M}embership-\underline{P}rivacy-preserving \underline{LoRA} (SMP-LoRA) method, which incorporates the MI gain into the denominator of the adaptation loss instead of directly summing it.
We theoretically demonstrate that this modification ensures a positive correlation (see Proposition~\ref{pro:spprivatelora} for details).
Specifically, the local smoothness (that is quantified by the Hessian norm) is positively correlated with and upper bounded by the gradient norm during adaptation, which can improve convergence.
Furthermore, we empirically corroborate that during SMP-LoRA, the Hessian norm is positively correlated with the gradient norm, as evidenced by the higher PCC (0.761) and the p-value of less than 0.001 in the lower panel of Figure~\ref{fig:fig1b_GradientHessian}.
The constrained local smoothness allows the SMP-LoRA to achieve better optimization, as shown in the blue dash-dot line of Figure~\ref{fig:fig1a_TrainingLoss}.

To evaluate the performance of the SMP-LoRA, we conducted adapting experiments using the Stable Diffusion v1.5~\citep{Compvis2024stable} on the Pokemon~\citep{Pinkney2022pokemon} and CelebA~\citep{Liu2015faceattributes} datasets, respectively.
Figure~\ref{fig:fig1d_Performance} shows that, although MP-LoRA (orange square marker) lowers the ASR to near-random levels, it significantly degrades the image generation capability of LoRA, as evidenced by a high FID score of 2.10 and the poor visual quality in Figure~\ref{fig:fig1c_GeneratedResults}.
In contrast, the SMP-LoRA (blue pentagon marker) effectively preserves membership privacy without sacrificing generated image quality significantly, as evidenced by its FID score of 0.32 and ASR of 51.97\%.

\section{Background and Preliminary}
This section outlines the related work and preliminary concepts in diffusion models, low-rank adaptation, and membership inference attacks.

\subsection{Diffusion Model}
Diffusion models (DMs)~\citep{Ho2020denoising,Song2020score} have shown remarkable performance in image synthesis.
Compared with other generative models such as GAN~\citep{Goodfellow2014GAN}, DMs can mitigate the problems of the training instability and the model collapse~\citep{Rombach2022high}, while achieving state-of-the-art results in numerous benchmarks~\citep{Dhariwal2021diffusion}.
Among various implementations of DMs, the Latent Diffusion Model (LDM) is renowned for generating high-quality images with limited computational resources, thereby widely utilized across many applications.
Fundamentally, LDM is a Denoising Diffusion Probabilistic Model (DDPM)~\citep{Ho2020denoising} built in the latent space, effectively reducing computational demand while enabling high-quality and flexible image generation~\citep{Rombach2022high}.
Therefore, LDMs have been widely utilized for adaptation~\citep{Gal2022image,Hu2021lora,Ruiz2023dreambooth}, capable of delivering high-performance models even when adapted on small-scale datasets.

Adapting LDMs involves a training process that progressively adds noise to the data and then learns to reverse the noise, finely tailoring the model's output.
To be specific, an image $x \in \mathcal{X}$ is initially mapped to a latent representation by a pre-trained encoder $\mathcal{E}: \mathcal{X} \rightarrow \mathcal{Z}$.
In the diffusion process, Gaussian noise $\epsilon \sim \mathcal{N}\left (0,1 \right) $ is progressively added at each time step $t = 1, 2, \ldots, T$, evolving $\mathcal{E}(x)$ into $z_{t} = \sqrt{\alpha_t }\mathcal{E}(x) + \sqrt{1 - \alpha _{t}}\epsilon$, where $\alpha _{t}\in [0,1]$ is a decaying parameter.
Subsequently, the model $f_{\theta } $ is trained to predict and remove noise $\epsilon$, therefore recovering $\mathcal{E}(x)$.
Building on this, a pre-trained decoder reconstructs the image from the denoised latent representation. 
Furthermore, to incorporate conditional information $y$ from various modalities, such as language prompts, a domain-specific encoder $\tau _{\phi }$ is introduced, projecting $y$ into an intermediate representation.
Given a pair $(x,y)$ consisting of an image $x$ and the corresponding text $y$, the adaptation loss for LDM is defined as follows:
\begin{equation}
\label{eq1_LDMLoss}
\ell_\adarm \left ( x, y; t, \epsilon, f_{\theta} \right ) = \left\| \epsilon - f_{\theta}(z_t, t, \tau_\phi(y)) \right\|_{2}^2 ,
\end{equation}
where $\tau_\phi$ refers to the pre-trained text encoder from CLIP~\citep{Radford2021learning}.

\subsection{Low-Rank Adaptation (LoRA)}
\label{sec2.2_LoRA}
To unleash the power of large pretrained models, many adaptation methods have emerged.
In particular, Low-Rank Adaptation (LoRA)~\citep{Hu2021lora} provides an efficient and effective solution by freezing the pre-trained model weights and introducing the trainable low-rank counterparts, significantly reducing the number of trainable parameters and memory usage during the LoRA adaptation process.
Therefore, LoRA not only lessens the demand for computational resources but also allows for the construction of multiple lightweight portable low-rank matrices on the same pre-trained LDM, addressing various downstream tasks~\citep{Pedro2023lorasd,Huggingface2023LoRA}.

Specifically, when adapting LDMs via LoRA, a low-rank decomposition is performed on each attention layer in the LDM backbone $f_{\theta}$. 
During LoRA, assuming that the original pre-trained weight $\theta \in \mathbb{R}^{d \times k}$, a trainable LoRA module $\bB \bA$ is randomly initialized and added to the pre-trained weights $\theta$, where $\mathbf{B} \in \mathbb{R}^{d \times r}$, $\mathbf{A} \in \mathbb{R}^{r \times k}$. Note that the rank $r$ is significantly less than $d$ or $k$, which ensures the computational efficiency of LoRA.
During the adaptation process via LoRA, for the augmented LDM backbone $f_{\bar{\theta} + \mathbf{B}\mathbf{A}}$, all trainable LoRA module parameters $\bB$ and $\bA$ are updated, while the original parameters $\bar{\theta}$ are frozen.
Given an image-test pair $(x,y)$, the adaptation loss during LoRA is formulated as follows:
\begin{equation}
\label{eq2_LoRALoss}
\ell_\adarm \left ( x, y ; f_{\tilde{\theta} + \bB\bA} \right ) = 
\left\| \epsilon - f_{\bar{\theta} + \bB\bA}(z_t, t, \tau_\phi(y)) \right\|_{2}^2.
\end{equation}
For notational simplicity, we omit the variables $t$ and $\epsilon$ in the adaptation loss $\ell_\adarm$.
Given the training dataset $\cD_\trrm = \{ (x_i,y_i) \}_{i=1}^n$ composed of $n \in \mathbb{N}^+$ image-text pairs, the training loss can be calculated as follows:
\begin{equation}
\label{eq3_LoRALossAverage}
\cL_{\adarm}(f_{\bar{\theta} + \bB\bA}, \cD_\trrm) = \frac{1}{n} \sum_{i=1}^n \ell_\adarm \left ( x_i, y_i; f_{\bar{\theta} + \bB\bA} \right ).
\end{equation}
Note that during the adaptation process via LoRA, the objective function that optimizes the parameters $\bB$ and $\bA$ is formulated as $ \min\limits_{\{\bB, \bA\}} \cL_{\adarm}(f_{\bar{\theta} + \bB\bA}, \cD_\trrm)$.

\subsection{Membership Inference Attack}
Membership Inference (MI) attack~\citep{Shokri2017membership} aims to determine whether a particular data point is part of a model's training set.
Recent studies~\citep{Carlini2023extracting,Duan2023diffusion} have shown that DMs are particularly vulnerable to MI attacks, thus undergoing high risks of privacy leakage.
MI attacks on diffusion models can be categorized based on the adversary's capabilities into white-box~\citep{Hu2023membership,Matsumoto2023membership,Pang2023white}, gray-box~\citep{Duan2023diffusion,Kong2023efficient,Fu2023probabilistic}, and black-box~\citep{Wu2022membership,Matsumoto2023membership,Pang2023black,Zhang2024generated} attacks.

In the black-box setting, \citet{Wu2022membership} were the first to explore MI attacks on DMs and achieved the highest success rate among the black-box attacks mentioned above.
They noted that DMs, when replicating training images, consistently produce outputs with higher fidelity and greater alignment with textual captions, indicating significant behavioural differences.
Therefore, they utilized the L2 distance between the embeddings of a given image and its corresponding caption-generated image to infer membership.

In the white-box setting, the gradient-based MI attack developed by \citet{Pang2023white} stands as the most effective method for DMs in terms of attack performance.
They leveraged the model's gradient to train an attack model for inference.
Additionally, \citet{Hu2023membership} and \citet{Matsumoto2023membership} employed a threshold-based MI attack by analyzing model loss at specific diffusion steps, which we refer to as the loss-based MI attack, also yielding significant attack performance.
In contrast, the black-box and gray-box MI attacks previously mentioned, which lack access to internal model information, achieve lower success rates compared to the white-box loss-based and gradient-based MI attacks~\citep{Pang2023white}.

Conventional techniques to defend against MI attacks include differential privacy~\citep{dwork2008differential,abadi2016deep}, min-max membership privacy game~\citep{Nasr2018machine}, data augmentation~\citep{devries2017improved,cubuk2020randaugment}, early stopping, etc.
To the best of our knowledge, these techniques have not been systematically applied or comprehensively evaluated in DMs.
Notably, some studies~\citep{dockhorn2022differentially,ghalebikesabi2023differentially,lyu2023differentially} have applied differential privacy to DMs to achieve a better privacy-accuracy trade-off, yet they have not evaluated its effectiveness against MI attacks in DMs.
Therefore, our paper takes the first step by developing a defensive adaptation based on the min-max membership privacy game to defend against MI attacks in DMs.


\section{Membership-Privacy-Preserving LoRA}
In this section, we first use the min-max optimization to formulate the learning objective of MP-LoRA. Then, we disclose the issue of unstable optimization of MP-LoRA. Finally, we propose the stable SMP-LoRA and its implementation. 

\subsection{A Vanilla Solution: MP-LoRA}
\label{sec3.1_PrivateLoRA}
\paragraph{Objective function.}
In MI attack, the conflicting objectives of defenders and adversaries can be modelled as a privacy game~\citep{Shokri2012protecting,Manshaei2013game,Alvim2017information}.
Adversaries can adjust their attack models to maximize MI gain against the target model, which requires that the defense can anticipate and withstand the strongest inference attacks.
Consequently, the defender's goal is to enhance the preservation of membership privacy in worst-case scenarios where the adversary achieves the maximum MI gain while maintaining the model performance.
Inspired by~\citet{Nasr2018machine}, we propose MP-LoRA to defend against MI attacks which is formulated as a min-max optimization problem as follows:
{\small
\begin{equation}
\label{eq4_PLoRAObjFun}
    \min_{\{\bB, \bA\}} \left ( 
    \underbrace{\cL_\adarm(f_{\bar{\theta} + \bB\bA}, \cD_\trrm)}_{\text{Adaptation loss}}
    + \lambda \underbrace{\max\limits_{\omega} G\left ( h_{\omega} , \cD_\auxrm, f_{\bar{\theta} + \bB\bA} \right )}
        _{\text{Membership inference gain}}
     \right ),
\end{equation}
}
where $\cL_\adarm(f_{\bar{\theta} + \bB\bA}, \cD_\trrm)$ refers to the adaptation loss for the LDM with LoRA module $f_{\bar{\theta} + \bB\bA}$ on the training dataset $\cD_\trrm$, $h_\omega$ is the proxy attack model parameterized by $\omega$, $G\left ( h_{\omega}, \cD_\auxrm, f_{\bar{\theta} + \bB\bA} \right )$ represents the MI gain of the proxy attack model $h_\omega$ on the auxiliary dataset $\cD_\auxrm$.

Therein, the inner maximization aims to search for the most effective proxy attack model $h_\omega$ for a given adapted LDM $f_{\bar{\theta} + \bB\bA}$ via maximizing the MI gain. 
The outer minimization, conversely, searches for the LDM $f_{\bar{\theta} + \bB\bA}$ that can best preserve membership privacy under the strong proxy attack model $h_\omega$ while being able to adapt on the training dataset.

\paragraph{Updating the proxy attack model in inner maximization.}
The proxy attack model $h_\omega$ equipped with white-box access to the target LDM $f_{\bar{\theta} + \bB\bA}$, aims to infer whether a specific image-text pair $(x, y)$ is from the training dataset $\cD_\trrm$ for adapting the target LDM $f_{\bar{\theta} + \bB\bA}$.
The model achieves this by constructing an auxiliary dataset $\cD_\auxrm$, which consists of half of the member data from $\cD_\trrm$, denoted as $\cD^\mrm_\auxrm$, and an equal amount of local non-member data $\cD^\nmrm_\auxrm$.
Using the auxiliary dataset $\cD_\auxrm$, $h_\omega$ trains a binary classifier based on the adaptation loss of the target LDM $f_{\bar{\theta} + \bB\bA}$ to predict the probability of $(x, y)$ for being a member of the $\cD_\trrm$.
Consequently, the MI gain of $h_\omega$ can be quantified based on its performance on the $\cD_\auxrm$ as follows:
\begin{align}
\label{eq5_MIGain}
& G\left ( h_{\omega} , \cD_\auxrm, f_{\bar{\theta} + \bB\bA} \right ) = \nonumber \\
& \frac{1}{2\left | \cD^\mrm_\auxrm \right | }\sum_{(x, y)\in \cD^\mrm_\auxrm}^{} \log_{}{\left (h_\omega\left ( \ell_\adarm\left ( x, y;f_{\bar{\theta} + \bB\bA} \right )   \right )\right )} \nonumber \\
& + \frac{1}{2\left | \cD^\nmrm_\auxrm \right | }\sum_{(x, y)\in \cD^\nmrm_\auxrm}^{} \log_{}{\left (1 - h_\omega\left ( \ell_\adarm\left ( x, y;f_{\bar{\theta} + \bB\bA} \right ) \right ) \right )}.
\end{align}
In the inner maximization, the proxy attack model optimizes the parameters $\omega$ by maximizing the MI gain, i.e., $\max\limits_{\omega}G\left ( h_{\omega}, \cD_\auxrm, f_{\bar{\theta} + \bB\bA} \right ) $.

\paragraph{Adapting the LDM in outer minimization.}
MP-LoRA optimizes the LDM by directly minimizing a weighted sum of the MI gain for the $h_\omega$ and the adaptation loss, which enables it to adapt to the training data and protect the private information of the training dataset simultaneously.
To be specific, the training loss of MP-LoRA is formulated as
\begin{equation}
\label{eq6_PLoRALoss}
\cL_\plrm = \cL_{\adarm}(f_{\bar{\theta} + \bB\bA}, \cD_\trrm) + \lambda \cdot G(h_{\omega }, \cD_\trrm ,f_{\bar{\theta} + \bB\bA}),
\end{equation}
where $\lambda \in \mathbb{R}$ controls the importance of optimizing the adaptation loss versus protecting membership privacy.
In the outer minimization of MP-LoRA, the parameters $\bB$ and $\bA$ is updated by minimizing the $\cL_\plrm$, i.e., $\min\limits_{\{ \bB, \bA \}} \cL_\plrm$.

MP-LoRA is realized by one step of inner maximization to obtain a power proxy attack model by maximizing the MI gain in Equation~\eqref{eq5_MIGain} and one step of outer minimization to update $\bA$ and $\bB$ by minimizing the training loss in Equation~\eqref{eq6_PLoRALoss}. The algorithm of MP-LoRA is shown in Algorithm~\ref{alg:algorithm2} (Appendix~\ref{app:algorithm2}).

\subsection{Unstable Issue of MP-LoRA}
\label{sec3.2_Issue}
In this subsection, we theoretically demonstrate that the convergence for MP-LoRA cannot be guaranteed due to unconstrained local smoothness. Then we validate the theoretical analyses with empirical evidence.

\begin{definition}[Relaxed Smoothness Condition from \citet{Zhang2019gradient}] \label{dif:relaxed_smoothness}
A second order differentiable function $f$ is $(L_0, L1)$-smooth if
\begin{equation}
\label{eq7_RelaxedSmoothnessCondition}
    \| \nabla^2 f(x) \| \le L_{0} + L_{1} \| \nabla f(x)\|.
\end{equation}
\end{definition}

\begin{lemma}[\citet{Zhang2019gradient}] \label{lem:relaxed_smoothness}
Let $f$ be a second-order differentiable function and $(L_0, L1)$-smooth. If the local smoothness, quantified by the Hessian norm (the norm of the Hessian matrix), is positively correlated with the gradient norm (i.e., $L_1 > 0$), then the gradient norm upper bounds the local smoothness, facilitating faster convergence and increasing the likelihood of converging to an optimal solution.
\end{lemma}

\begin{proposition}
\label{pro:privatelora}
MP-LoRA does not satisfy the positive correlation as described in Lemma~\ref{lem:relaxed_smoothness}, therefore the convergence cannot be guaranteed and the model may settle at a suboptimal solution.
\end{proposition}

\begin{proof}
We establish the Relaxed Smoothness Condition for MP-LoRA as follows:
\begin{align}
\label{eq8_PLoRA_RSC}
    & \|\frac{\partial^2 \mathcal{L}_{\mathrm{PL}}}{\partial \mathbf{B}\mathbf{A}^2}\| \le L_{0} + L_{1} \|\frac{\partial \mathcal{L}_{\mathrm{PL}}}{\partial \mathbf{B}\mathbf{A}}\|, \nonumber \\
    & \mathrm{where}~~L_{0} = \|\frac{\partial^2 \mathcal{L}_{\mathrm{ada}}}{\partial \mathbf{B}\mathbf{A}^2}\| + \lambda \|\frac{\partial^2 G}{\partial \mathbf{B}\mathbf{A}^2}\|,~L_{1} = 0,
\end{align}
in which $\cL_\adarm$ represents the adaptation loss and $G$ represents the MI gain. The detailed derivation is presented in Appendix~\ref{app:SecondDerivative}. The value of $L_1$ being zero indicates that the Hessian norm is independent of and not bounded by the gradient norm, suggesting that the local smoothness is unconstrained.
\end{proof}

Next, we provide empirical evidence to support our theoretical analyses. We tracked the gradient norm and the Hessian norm of the training loss at each training iteration, and calculated their Pearson Correlation coefficient (PCC) and p-value as shown in Figure~\ref{fig:fig1b_GradientHessian}. The details for calculating the gradient norm and the Hessian norm can be found in Appendix~\ref{app:CalGradHess}. In Figure~\ref{fig:fig1b_GradientHessian}, the low PPC of 0.043 for MP-LoRA suggests a very weak correlation between the Hessian norm and the gradient norm. Additionally, with the p-value of 0.052, there is insufficient evidence to reject the hypothesis of no correlation. This indicates that the Hessian norm is unbounded, implying that the local smoothness, quantified by the Hessian norm~\citep{Bubeck2015convex}, is unconstrained. Such unconstrained local smoothness leads to the unstable optimization issue in MP-LoRA, and even to the failure of adaptation, as evidenced in the orange dashed line of Figure~\ref{fig:fig1a_TrainingLoss} and the poor visual quality of the generated images in Figure~\ref{fig:fig1c_GeneratedResults}.

\subsection{Stabilizing MP-LoRA}
\label{sec3.3_Stabilizing}
To mitigate the aforementioned optimization issue of MP-LoRA, we propose SMP-LoRA by incorporating the MI gain into the denominator of the adaptation loss. The objective function of SMP-LoRA is formulated as follows:
\begin{equation}
\label{eq9_SPLoRAObjFun}
    \min_{\{\bB, \bA\}} \left ( 
    \frac{\cL_\adarm(f_{\bar{\theta} + \bB\bA}, \cD_\trrm)}{
    1 - \lambda \max\limits_{\omega} G\left ( h_{\omega} , \cD_\auxrm, f_{\bar{\theta} + \bB\bA} \right ) 
    } \right ).
\end{equation}
To optimize Equation~\eqref{eq9_SPLoRAObjFun}, SMP-LoRA targets to minimize the following training loss function, i.e.,
\begin{equation}
\label{eq10_SPLoRALoss}
\cL_\splrm = \frac{\cL_{\adarm}(f_{\bar{\theta} + \bB\bA}, \cD_\trrm)}{1 - \lambda \cdot G\left ( h_{\omega} , \cD_\trrm, f_{\bar{\theta} + \bB\bA} \right ) + \delta } ,
\end{equation}
where $\delta$ is a stabilizer with a small value such as $1e-5$. This prevents the denominator from approaching zero and ensures stable calculation.

The implementation of SMP-LoRA is detailed in Algorithm~\ref{alg:algorithm1}. At each training step, SMP-LoRA will first update the proxy attack model by maximizing the MI gain and then update the LDM by minimizing the training loss $\cL_\splrm$.

\begin{algorithm}[tb]
\caption{Stable Membership-Privacy-preserving LoRA}
\label{alg:algorithm1}
\textbf{Input}: Training dataset $\cD_\trrm$ for adaptation process, Auxiliary dataset $\cD_\auxrm = \cD^\mrm_\auxrm \cup \cD^\nmrm_\auxrm$, a pre-trained LDM $f_{\theta}$, a proxy attack model $h_{\omega}$ parameterized by $\omega$, learning rate $\eta_1$ and $\eta_2$ \\
\textbf{Output}: a SMP-LoRA for LDMs
\begin{algorithmic}[1] 
\STATE Perform low-rank decomposition on $f_{\theta}$ to obtain $f_{\bar{\theta} + \bB\bA}$ ($\bB$ and $\bA$ are trainable LoRA modules)
\FOR{each epoch}
\FOR{each training iteration}
\STATE Sample batches $S^\mrm$ and $S^\nmrm$ from $\cD^\mrm_\auxrm$ and $\cD^\nmrm_\auxrm$
\STATE Calculate the MI gain $G^*$ on $S^\mrm \cup S^\nmrm$
\STATE Update the parameters $\omega \gets \omega + \eta_1 \cdot \nabla_\omega G^*$.
\STATE Sample a fresh batch from $\cD_\trrm$
\STATE Calculate the training loss $\cL^* = \cL_\splrm$
\STATE Update parameters $\bA \gets \bA - \eta_2 \cdot \nabla_\bA \cL^*$ and $\bB \gets \bB - \eta_2 \cdot \nabla_\bB \cL^*$, respectively
\ENDFOR
\ENDFOR
\end{algorithmic}
\end{algorithm}

\begin{proposition}
\label{pro:spprivatelora}
SMP-LoRA satisfies the positive correlation as described in Lemma~\ref{lem:relaxed_smoothness}, thus promoting faster convergence, and the model is more likely to converge to an optimal solution.
\end{proposition}

\begin{proof}
We establish the Relaxed Smoothness Condition for SMP-LoRA as follows:
\begin{align}
\label{eq11_SPLoRA_RSC}
& \|\frac{\partial^2 \mathcal{L}_{\mathrm{SPL}}}{\partial \mathbf{B}\mathbf{A}^2}\| \le L^\prime_{0} + L^\prime_{1}\|\frac{\partial \mathcal{L}_{\mathrm{SPL}}}{\partial \mathbf{B}\mathbf{A}}\|, \nonumber \\
& \mathrm{where}~~\mu = \frac{\partial \cL_\adarm}{\partial \bB\bA},~\nu = \lambda \frac{\partial G}{\partial \bB\bA}, \nonumber \\
& L^\prime_{0} = \frac{1}{1-\lambda G + \delta} \cdot \|\frac{\partial^2 \cL_\adarm}{\partial \bB\bA^2}\| + \frac{\lambda \cL_\adarm}{(1-\lambda G + \delta)^2}\cdot \|\frac{\partial G^2}{\partial \bB\bA^2}\|, \nonumber \\
& L^\prime_{1} = \frac{2\|\nu\|}{1-\lambda G + \delta}. 
\end{align}
Please refer to Appendix~\ref{app:SecondDerivative} for detailed derivation. The value of $L^\prime_{1}$ being greater than zero indicates that the Hessian norm is positively correlated with and upper bounded by the gradient norm, suggesting that the gradient norm constrains the local smoothness during adaptation. 
\end{proof}

\begin{table*}[t]
  \caption{Performance of LoRA, MP-LoRA, and SMP-LoRA across five datasets, as measured by FID, KID, ASR, AUC, and TPR at $5\%$ FPR. Results are presented as mean $\pm$ standard error, based on three independent runs with different seeds. Due to the limited size of the CelebA$\_$Small, CelebA$\_$Gender, and CelebA$\_$Varying datasets, the FID scores are not available. It is important to note that an AUC value closer to 0.5 (random level) indicates a stronger defensive capability against MI attacks.}
  \label{tab:tab1_Performance}
  \centering
  \begin{tabular}{llccccc}
    \toprule
    Dataset & Method & FID \textbf{$\downarrow$} & KID \textbf{$\downarrow$} & ASR (\%) \textbf{$\downarrow$} & $\frac{\left | \mathrm{AUC} - 0.5 \right |}{0.5}$ \textbf{$\downarrow$} & \begin{tabular}[c]{@{}l@{}}TPR\\ @$5\%$FPR (\%) \textbf{$\downarrow$}\end{tabular} \\
    \midrule
    \multirow{3}{*}{Pokemon} & LoRA  & \textbf{0.20$\pm$0.04} & \textbf{0.003$\pm$0.00} & 82.27$\pm$4.38 & 0.73$\pm$0.09 & 4.44$\pm$1.45 \\
    & MP-LoRA & 2.10$\pm$0.51 & 0.121$\pm$0.00 & \textbf{51.67$\pm$2.73} & \textbf{0.07$\pm$0.04} & \textbf{2.23$\pm$1.27} \\
    & SMP-LoRA &  0.32$\pm$0.07 & 0.004$\pm$0.00 & 51.97$\pm$1.20 & 0.14$\pm$0.02 & 4.45$\pm$2.12 \\
    \midrule
    \multirow{3}{*}{CelebA$\_$Small} & LoRA & N/A & 0.05$\pm$0.00 & 91.53$\pm$2.27 & 0.94$\pm$0.02 & 81.33$\pm$9.68 \\
    & MP-LoRA & N/A & 0.30$\pm$0.04 & \textbf{55.67$\pm$2.91} & \textbf{0.12$\pm$0.05} & \textbf{8.00$\pm$3.46} \\
    & SMP-LoRA & N/A & \textbf{0.03$\pm$0.01} & 56.00$\pm$2.52 & 0.24$\pm$0.08 & 17.33$\pm$5.21 \\
    \midrule
    \multirow{3}{*}{CelebA$\_$Large} & LoRA & \textbf{0.52$\pm$0.01} & 0.06$\pm$0.00 & 87.83$\pm$0.17 & 0.87$\pm$0.02 & 66.83$\pm$6.00 \\
    & MP-LoRA & 2.34$\pm$0.71 & 0.30$\pm$0.05 & 53.58$\pm$1.52 & \textbf{0.07$\pm$0.04} & 2.50$\pm$0.87 \\
    & SMP-LoRA & 0.60$\pm$0.04 & \textbf{0.05$\pm$0.00} & \textbf{48.83$\pm$1.17} & 0.19$\pm$0.05 & \textbf{1.67$\pm$1.01} \\
    \midrule
    \multirow{3}{*}{CelebA$\_$Gender} & LoRA & N/A & \textbf{0.06$\pm$0.00} & 84.63$\pm$1.67 & 0.79$\pm$0.06 & 36.67$\pm$5.83 \\
    & MP-LoRA & N/A & 0.32$\pm$0.04 & 55.43$\pm$2.07 & \textbf{0.14$\pm$0.06} & 5.83$\pm$1.67 \\
    & SMP-LoRA & N/A & \textbf{0.06$\pm$0.00} & \textbf{54.20$\pm$0.40} & 0.15$\pm$0.05 & \textbf{4.17$\pm$3.00} \\
    \midrule
    \multirow{3}{*}{CelebA$\_$Varying} & LoRA & N/A & 0.06$\pm$0.00 & 87.30$\pm$0.92 & 0.83$\pm$0.04 & 47.10$\pm$23.37  \\
    & MP-LoRA & N/A & 0.26$\pm$0.06 & 55.73$\pm$0.75 & \textbf{0.07$\pm$0.04} & \textbf{3.73$\pm$1.07}  \\
    & SMP-LoRA & N/A & \textbf{0.04$\pm$0.01} & \textbf{53.97$\pm$1.82} & 0.15$\pm$0.01 & 4.57$\pm$3.27  \\
    \bottomrule
  \end{tabular}
\end{table*}

Subsequently, we further corroborate our theoretical analyses with the following empirical evidence.
Compared to MP-LoRA's insignificant correlation, SMP-LoRA demonstrates a strong positive correlation between the Hessian norm and the gradient norm, evidenced by the PCC of 0.761 and the p-value less than 0.001 in the lower panel of Figure~\ref{fig:fig1b_GradientHessian}.
This indicates that the Hessian norm, which represents the local smoothness, is upper bounded by the gradient norm, resulting in lower mean (0.105) and standard deviation (0.253) than MP-LoRA.
Consequently, the constrained local smoothness mitigates the issue of unstable optimization and enables the SMP-LoRA to converge to a more optimal solution, as demonstrated by the progressively decreasing training loss shown in the blue dash-dot line of Figure~\ref{fig:fig1a_TrainingLoss} and the superior performance on both FID and ASR metrics illustrated by the blue pentagon marker in Figure~\ref{fig:fig1d_Performance}.

Notably, SMP-LoRA also exhibits lower mean and standard deviation of the gradient norm compared to MP-LoRA. We provide further empirical analysis in Appendix~\ref{app:GradScale}, showing that SMP-LoRA can implicitly rescale the gradient during adaptation by introducing the factors $ \frac{1}{1 - \lambda G + \delta} $ and $\frac{\cL_{\adarm}}{\left ( 1 - \lambda G + \delta \right )^2}$, thereby leading to more stable gradient and controlled gradient scale compared to MP-LoRA.

\section{Experiments}
\label{sec4_Exp}
In this section, we first evaluate the performance of LoRA, MP-LoRA, and SMP-LoRA in terms of image generation capability and effectiveness in defending against MI attacks.
Then, we conduct ablation studies on the important hyperparameters and further extend our membership-privacy-preserving method to full fine-tuning and DreamBooth~\citep{Ruiz2023dreambooth} methods.
Subsequently, we compare SMP-LoRA with traditional techniques for stabilizing the gradient and evaluate the effectiveness of SMP-LoRA in defending against MI attacks in different settings.

\begin{table*}[t]
  \caption{Performance of LoRA and SMP-LoRA across three larger datasets, as measured by FID, KID, AUC, and TPR.}
  \label{tab:tab2_Performance}
  \small
  \centering
  \begin{tabular}{llccccc}
    \toprule
    Dataset & Method & FID \textbf{$\downarrow$} & KID \textbf{$\downarrow$} & ASR (\%) \textbf{$\downarrow$} & $\frac{\left | \mathrm{AUC} - 0.5 \right |}{0.5}$ \textbf{$\downarrow$} & \begin{tabular}[c]{@{}l@{}}TPR\\ @$5\%$FPR (\%) \textbf{$\downarrow$}\end{tabular} \\
    \midrule
    \multirow{2}{*}{CelebA$\_$Large$\_$5X} & LoRA  & \textbf{0.53} & \textbf{0.051} & 92.80 & 0.94 & 85.80 \\
    & SMP-LoRA &  0.59 & 0.062 & \textbf{51.85} & \textbf{0.15} & \textbf{1.70} \\
    \midrule
    \multirow{2}{*}{AFHQ$\_$Large$\_$5X} & LoRA & \textbf{0.39} & \textbf{0.025} & 88.20 & 0.86 & 85.00 \\
    & SMP-LoRA & 0.51 & 0.041 & \textbf{56.00} & \textbf{0.26} & \textbf{12.00} \\
    \midrule
    \multirow{2}{*}{MS-COCO$\_$Large$\_$5X} & LoRA & \textbf{0.37} & \textbf{0.014} & 80.40 & 0.78 & 68.50 \\
    & SMP-LoRA & 0.72 & 0.023 & \textbf{46.10} & \textbf{0.14} & \textbf{4.1} \\
    \bottomrule
  \end{tabular}
\end{table*}

\paragraph{Dataset.}
In our experiment, we utilized four datasets: Pokemon~\citep{Pinkney2022pokemon}, CelebA~\citep{Liu2015faceattributes}, AFHQ~\citep{choi2020afhq}, and MS-COCO~\citep{lin2014mscoco}. We created several subsets from CelebA, including CelebA$\_$Small and CelebA$\_$Large, both balanced with equal image contribution per individual, as well as CelebA$\_$Gender and CelebA$\_$Varying, which are imbalanced with a $7:3$ gender ratio and varied image contributions per individual, respectively. Additiaonlly, we constructed CelebA$\_$Large$\_$5X, which is five times larger than the CelebA$\_$Large along with comparably sized AFHQ$\_$Large$\_$5X and MS-COCO$\_$Large$\_$5X. For more details, please refer to Appendix~\ref{app:ExpSetup}.

\paragraph{Model hyperparameters.} We utilized the official pre-trained Stable Diffusion v1.5~\citep{Compvis2024stable} for LDMs to build the LoRA module, with specific model hyperparameters detailed in Table~\ref{tab:tab5_Hyperparameter} in Appendix~\ref{app:ExpSetup}. We employed a 3-layer MLP as the proxy attack model $h_\omega$ and a structurally similar new attack model $h^\prime$ to evaluate the effectiveness of adapted LDMs in defending against MI attacks.

\paragraph{Evaluation metrics.} We employed Attack Success Rate (ASR)~\citep{Choquette2021label}, Area Under the ROC Curve (AUC), and True Positive Rate (TPR) to evaluate the effectiveness of MP-LoRA and SMP-LoRA in defending against MI attacks.
Lower values for ASR and TPR indicate a more effective defense against MI attacks. 
Consistent with many prior studies~\citep{Chen2022relaxloss,Ye2022enhanced,Duan2023diffusion,Dubinski2024towards}, we default that the TPR measures the capability of the attack model to identify samples as members correctly.
Consequently, an AUC value closer to 0.5, i.e., a smaller value of $\frac{\left | \mathrm{AUC} - 0.5 \right |}{0.5}$, represents a stronger defense capability, as MI attacks involve determining both membership and non-membership.
For assessing the image generation capability of the adapted LDMs, we utilized the Fr\'{e}chet Inception Distance (FID)~\citep{Heusel2017gans} and the Kernel Inception Distance (KID)~\citep{Binkowski2018demystifying}, with lower values denoting better image quality. Please refer to Appendix~\ref{app:ExpSetup} for detailed explanations.

\subsection{Effectiveness of SMP-LoRA in Defending Against MI Attacks}
\label{sec4.1_Performance}
In Table~\ref{tab:tab1_Performance}, we report the performance of LoRA, MP-LoRA, and SMP-LoRA, across the Pokemon, CelebA$\_$Small, CelebA$\_$Large, CelebA$\_$Gender, CelebA$\_$Varying datasets.
Compared to LoRA, MP-LoRA displays considerably higher FID and KID scores, indicating lower quality of generated images.
Meanwhile, MP-LoRA exhibits near-random levels of ASR and AUC, along with low TPR values at 5$\%$ TPR, showcasing its effectiveness in defending against MI attacks.
Notably, the FID and KID scores for SMP-LoRA closely align with those of LoRA, suggesting that SMP-LoRA only makes a minor sacrifice to the quality of generated images. 
Also, SMP-LoRA achieves near-random levels of ASR and AUC, and low TPR values at 5$\%$ FPR.
These results demonstrate that compared to MP-LoRA, SMP-LoRA effectively preserves membership privacy against MI attacks without significantly compromising image generation capability.

In Table~\ref{tab:tab2_Performance}, we present the performance of LoRA and SMP-LoRA on the CelebA$\_$Large$\_$5X, AFHQ$\_$Large$\_$5X, and MS-COCO$\_$Large$\_$5X datasets. SMP-LoRA remains effective on larger datasets, consistently defending against MI attacks and generating high-quality images.

\subsection{Ablation Study}
\label{sec4.2_AblationStudy}
This subsection presents ablation studies on the important hyperparameters: the coefficient $\lambda$, the learning rate $\eta_{2}$, and the LoRA's rank $r$.
We also extended the application of SMP-LoRA to the full fine-tuning and DreamBooth~\citep{Ruiz2023dreambooth} method to assess the generalizability of our membership-privacy-preserving method.
Additionally, we compared the performance of SMP-LoRA with gradient clipping and normalization techniques.
Furthermore, we evaluated the effectiveness of SMP-LoRA in preserving membership privacy under the black-box MI attacks~\citep{Wu2022membership} and the white-box gradient-based MI attacks~\citep{Pang2023white}.
Due to space constraints, only the key conclusions are presented in the main paper, with all tables and detailed analyses located in Appendix~\ref{app:detailed_as}.

\paragraph{Coefficient $\lambda$.}
Table~\ref{tab:tab6_aba_lambda} (Appendix~\ref{app:detailed_as}) presents the performance of SMP-LoRA with different coefficient $\lambda \in \{1.00, 0.50, 0.10, 0.05, 0.01\}$ across the Pokemon, CelebA$\_$Small, and CelebA$\_$Large datasets. 
As $\lambda$ decreases from 1.00 to 0.01, the FID and KID scores gradually decrease, while ASR increases and AUC deviates further from 0.5, suggesting that a lower $\lambda$ shifts the focus more towards minimizing adaptation loss rather than protecting membership privacy.

Figure~\ref{fig:fig3} in Appendix~\ref{app:detailed_as} shows the ROC curves for SMP-LoRA with these $\lambda$ values across all three datasets.
SMP-LoRA effectively defends against MI attacks under both strict False Positive Rate (FPR) constraints and more lenient error tolerance conditions.


\paragraph{Learning rate $\eta_{2}$.}
Table~\ref{tab:tab7_LR} in Appendix~\ref{app:detailed_as} displays the performance of SMP-LoRA with different learning rates $\eta_{2} \in \{1e-4, 1e-5, 1e-6\}$ on the Pokemon dataset.
SMP-LoRA consistently preserves membership privacy across all tested learning rates.

\paragraph{LoRA's rank $r$.}
Table~\ref{tab:tab8_Rank} in Appendix~\ref{app:detailed_as} shows the performance of SMP-LoRA with different rank $r \in \{128, 64, 32, 16, 8\}$ on the Pokemon dataset.
The performance of SMP-LoRA is not significantly affected by the LoRA's rank $r$.

\paragraph{Extending to the full fine-tuning and DreamBooth Methods.}
In Table~\ref{tab:tab9_finetuning} (Appendix~\ref{app:detailed_as}), we present the performance of SMP-LoRA and its extension to the full fine-tuning and DreamBooth~\citep{Ruiz2023dreambooth} methods on the Pokemon dataset. Our membership-privacy-preserving method continues to effectively protect membership privacy when applied to these methods, highlighting its potential applicability across different adaptation methods. 

\paragraph{Comparing with gradient clipping and normalization techniques.}
In Table~\ref{tab:tab10_ClipNorm} (Appendix~\ref{app:detailed_as}), we report the performance of SMP-LoRA and MP-LoRA enhanced with gradient clipping and normalization techniques on the Pokemon dataset.
These traditional techniques for stabilizing the gradient cannot address the unstable optimization issue in MP-LoRA.

\paragraph{Defending against MI attacks in different settings.}
In Table~\ref{tab:tab11_BlackWhite} (Appendix~\ref{app:detailed_as}), we display the attack performance on LoRA and SMP-LoRA using the black-box MI attack~\citep{Wu2022membership} and the white-box gradient-based MI attack~\citep{Pang2023white}, which is currently the most potent MI attack targeting DMs. This renders further comparisons with weaker attacks unnecessary, such as gray-box MI attacks~\citep{Duan2023diffusion,Kong2023efficient,Fu2023probabilistic}.
Compared to LoRA, SMP-LoRA, specifically designed to defend against white-box loss-based MI attacks, consistently provides enhanced membership privacy protection when facing MI attacks in different settings.
Implementation details for the black-box and white-box gradient-based MI attacks are available in Appendix~\ref{app:BlackboxImple}.

\section{Conclusion}
In this paper, we proposed membership-privacy-preserving LoRA (MP-LoRA), a method based on low-rank adaptation (LoRA) for adapting latent diffusion models (LDMs), while mitigating the risk of privacy leakage.
We first highlighted the unstable issue in MP-LoRA. Directly minimizing the sum of the adaptation loss and MI gain can lead to unconstrained local smoothness, which results in unstable optimization.
To mitigate this issue, we further proposed a stable membership-privacy-preserving LoRA (SMP-LoRA) method, which constrains the local smoothness through the gradient norm to improve convergence.
Detailed theoretical analyses and comprehensive empirical results demonstrate that the SMP-LoRA can effectively preserve membership privacy against MI attacks and generate high-quality images.

\section*{Acknowledgments}
Feng Liu is supported by the Australian Research Council (ARC) with grant numbers DP230101540 and DE240101089, and the NSF\&CSIRO Responsible AI program with grant number 2303037.

Di Wang is supported in part by the funding BAS/1/1689-01-01, URF/1/4663-01-01,  REI/1/5232-01-01, REI/1/5332-01-01, and URF/1/5508-01-01 from KAUST, and funding from KAUST - Center of Excellence for Generative AI, under award number 5940.

\bibliography{Reference}

\clearpage
\appendix
\section{Algorithm of MP-LoRA}
\label{app:algorithm2}
We provide the specific implementation of MP-LoRA as follows:
\begin{algorithm}[h]
\caption{Membership-Privacy-preserving LoRA}
\label{alg:algorithm2}
\textbf{Input}: Training dataset $\cD_\trrm$ for adaptation process, Auxiliary dataset $\cD_\auxrm = \cD^\mrm_\auxrm \cup \cD^\nmrm_\auxrm$, a pre-trained LDM $f_{\theta}$, a proxy attack model $h_{\omega}$ parameterized by $\omega$, learning rate $\eta_1$ and $\eta_2$ \\
\textbf{Output}: a MP-LoRA for LDMs
\begin{algorithmic}[1] 
\STATE Perform low-rank decomposition on $f_{\theta}$ to obtain $f_{\bar{\theta} + \bB\bA}$ ($\bB$ and $\bA$ are trainable LoRA modules)
\FOR{each epoch}
\FOR{each training iteration}
\STATE Sample batches $S^\mrm$ and $S^\nmrm$ from $\cD^\mrm_\auxrm$ and $\cD^\nmrm_\auxrm$
\STATE Calculate the MI gain $G^*$ on $S^\mrm \cup S^\nmrm$
\STATE Update the parameters $\omega \gets \omega + \eta_1 \cdot \nabla_\omega G^*$.
\STATE Sample a fresh batch from $\cD_\trrm$
\STATE Calculate the training loss $\cL^* = \cL_\plrm$
\STATE Update parameters $\bA \gets \bA - \eta_2 \cdot \nabla_\bA \cL^*$ and $\bB \gets \bB - \eta_2 \cdot \nabla_\bB \cL^*$, respectively
\ENDFOR
\ENDFOR
\end{algorithmic}
\end{algorithm}

\begin{figure*}[t]
    \centering
    \begin{minipage}{0.48\textwidth}
        \begin{subfigure}[t]{\linewidth}
            \centering
            \includegraphics[width=0.9\linewidth,keepaspectratio]{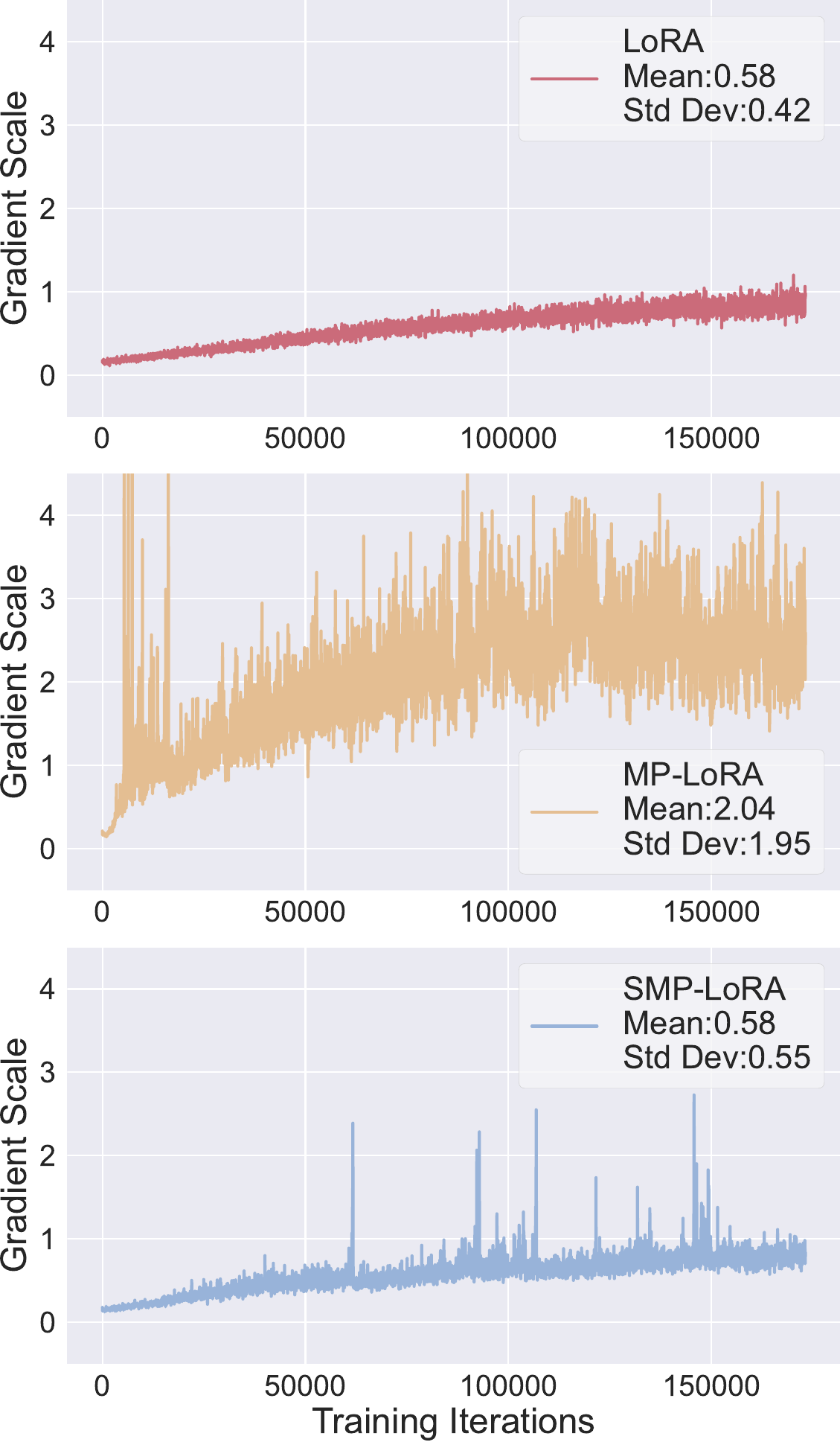}
            \caption{Gradient scale obtained by training loss.}
            \label{fig:fig2a_GradScale}
        \end{subfigure}
    \end{minipage}%
    \begin{minipage}{0.48\textwidth}
        \begin{subfigure}[t]{\linewidth}
            \centering
            \includegraphics[width=0.8\linewidth,height=\linewidth,keepaspectratio]{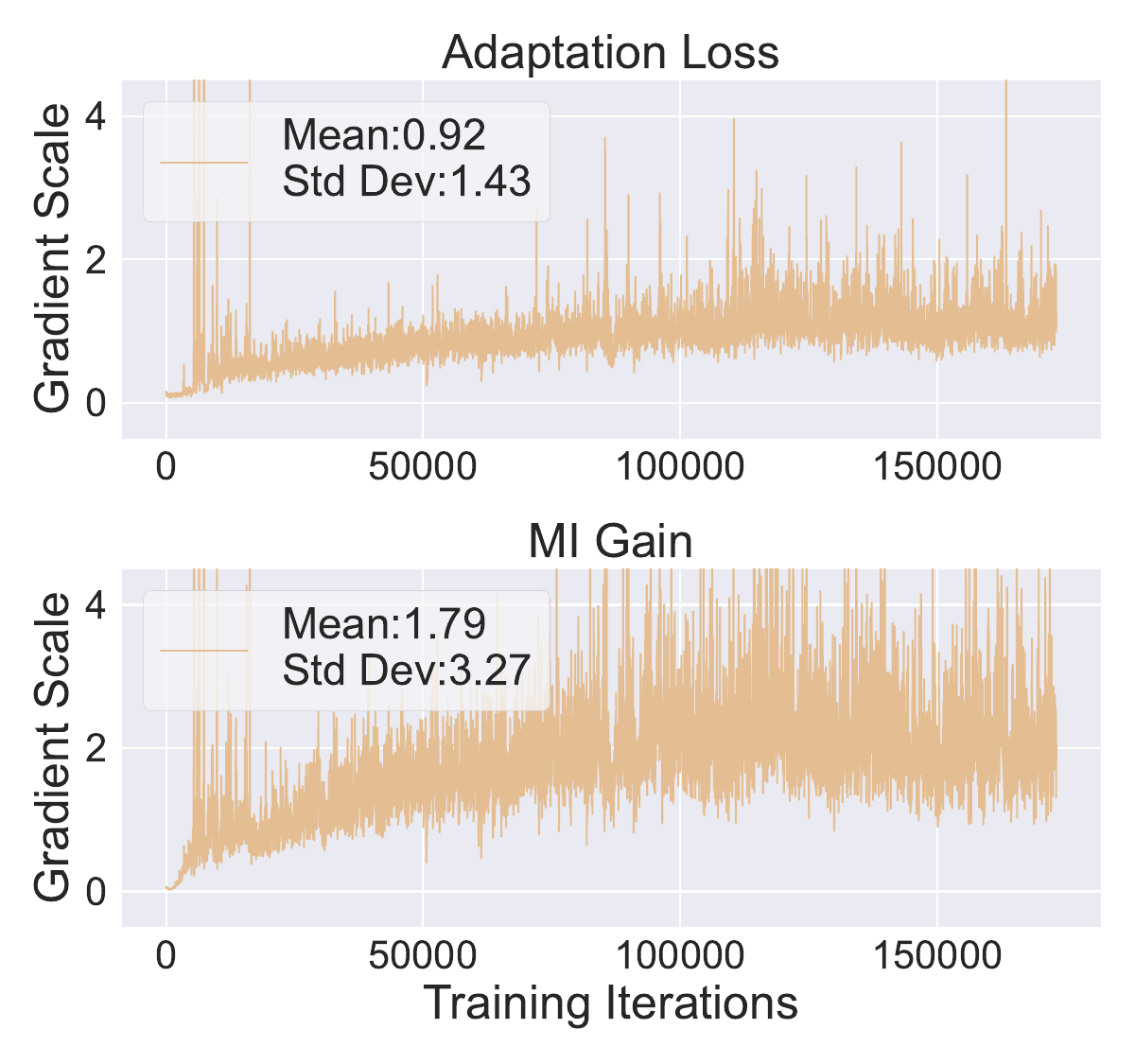}
            \caption{Gradient scale of MP-LoRA.}
            \label{fig:fig2b_PLoRAGradScale}
        \end{subfigure}\\ 
        \begin{subfigure}[b]{\linewidth}
            \centering
         \includegraphics[width=0.8\linewidth,height=\linewidth,keepaspectratio]{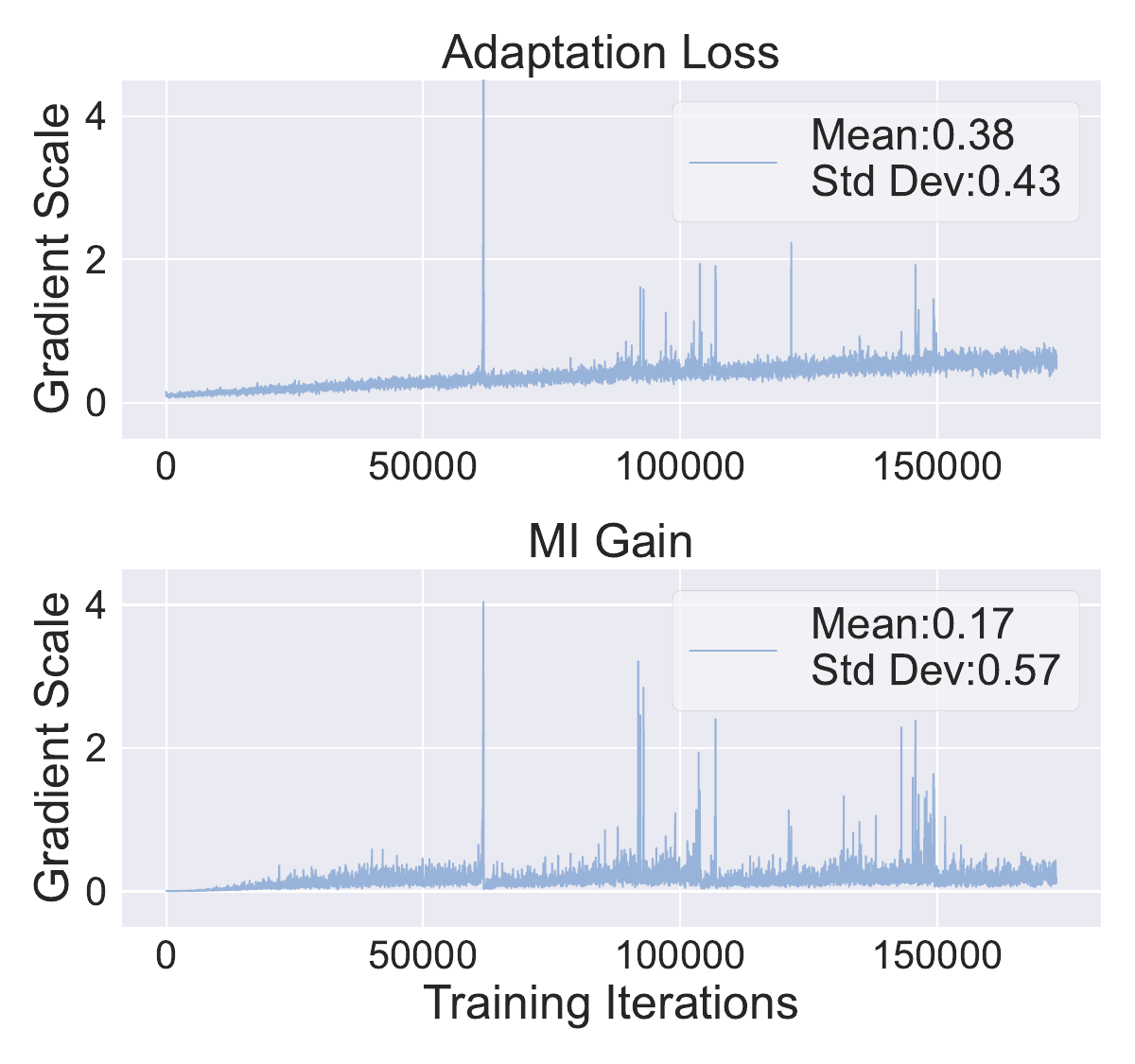} 
            \caption{Gradient scale of SMP-LoRA.}
            \label{fig:fig2c_SPLoRAGradScale}
        \end{subfigure}
    \end{minipage}
    \caption{Figure~\ref{fig:fig2a_GradScale} shows the mean and standard deviation of gradient scales obtained by training loss throughout the training iterations of LoRA, MP-LoRA, and SMP-LoRA on the Pokemon dataset. Figure~\ref{fig:fig2b_PLoRAGradScale} and~\ref{fig:fig2c_SPLoRAGradScale} display the gradient scales obtained by the adaptation loss and MI gain respectively during MP-LoRA and SMP-LoRA. Compared with MP-LoRA, SMP-LoRA has more stable gradient and controlled gradient scale.}
    \label{fig:fig2} 
\end{figure*}

\section{Detailed Derivation for Establishing the Relaxed Smoothness Condition}
\label{app:SecondDerivative}
We provide the detailed derivation for establishing the relaxed smoothness condition for MP-LoRA and SMP-LoRA, respectively.

For MP-LoRA, we derive the second derivatives of the training loss $\cL_\plrm$ in Equation~\eqref{eq6_PLoRALoss} to establish the relaxed smoothness condition.
\begin{align}
\label{eq:PLoRA_SecDer}
\frac{\partial^2 \cL_\plrm}{\partial \bB\bA^2} & = \frac{\partial}{\partial \bB\bA}\left( \frac{\partial \cL_\plrm}{\partial \bB\bA} \right) \nonumber \\ & = \frac{\partial }{\partial \bB\bA}\left (\frac{\partial \cL_\adarm}{\partial \bB\bA} + \lambda \frac{\partial G}{\partial \bB\bA} \right) \nonumber \\ & = \frac{\partial^2 \cL_\adarm}{\partial \bB\bA^2} + \lambda \frac{\partial^2 G}{\partial \bB\bA^2},
\end{align}
leading to:
\begin{align}
\left\|\frac{\partial^2 \mathcal{L}_{\mathrm{PL}}}{\partial \mathbf{B}\mathbf{A}^2}\right\|
\le \left\|\frac{\partial^2 \cL_\adarm}{\partial \bB\bA^2}\right\| + \left\|\lambda \frac{\partial^2 G}{\partial \bB\bA^2}\right\| + 0\cdot\left\|\frac{\partial \mathcal{L}_{\mathrm{PL}}}{\partial \mathbf{B}\mathbf{A}}\right\|, \nonumber
\end{align}
resulting in:
\begin{align}
& \|\frac{\partial^2 \mathcal{L}_{\mathrm{PL}}}{\partial \mathbf{B}\mathbf{A}^2}\| \le L_{0} + L_{1} \|\frac{\partial \mathcal{L}_{\mathrm{PL}}}{\partial \mathbf{B}\mathbf{A}}\|, \nonumber \\
& \mathrm{where}~~L_{0} = \|\frac{\partial^2 \mathcal{L}_{\mathrm{ada}}}{\partial \mathbf{B}\mathbf{A}^2}\| + \lambda \|\frac{\partial^2 G}{\partial \mathbf{B}\mathbf{A}^2}\|,~L_{1} = 0, \nonumber
\end{align}
in which $\cL_\adarm$ represents the adaptation loss and $G$ represents the MI gain.

For SMP-LoRA, we derive the second derivative of the training loss $\cL_\splrm$ in Equation~\eqref{eq10_SPLoRALoss} to establish the relaxed smoothness condition.
{\scriptsize
\begin{align}
\label{eq:SPLoRA_SecDer}
& \frac{\partial^2 \cL_\splrm}{\partial \bB\bA^2}  = \frac{\partial}{\partial \bB\bA}\left( \frac{\partial \cL_\splrm}{\partial \bB\bA}\right) \nonumber \\ & = \frac{\partial}{\partial \bB\bA} \left ( \frac{\left ( 1 - \lambda G + \delta \right )\frac{\partial \cL_{\adarm}}{\partial \bB\bA} - \cL_{\adarm} \frac{\partial \left ( 1 - \lambda G + \delta \right )}{\partial \bB\bA} }{\left ( 1 - \lambda G + \delta \right )^2}  \right ) \nonumber \\ 
& = \frac{\partial}{\partial \bB\bA} \cdot \left[ \frac{1}{1-\lambda G + \delta} \cdot \frac{\partial \cL_\adarm}{\partial \bB\bA} + \frac{\cL_\adarm}{(1-\lambda G + \delta)^2} \cdot \lambda\frac{\partial G}{\partial \bB\bA} \right] \nonumber \\ 
& = \frac{\partial}{\partial \bB\bA} \left( \frac{1}{1-\lambda G + \delta} \right) \cdot \frac{\partial \cL_\adarm}{\partial \bB\bA} + \frac{1}{1-\lambda G + \delta} \cdot \frac{\partial^2 \cL_\adarm}{\partial \bB\bA^2} \nonumber \\ \nonumber
& \quad \quad + \frac{\partial}{\partial \bB\bA} \left[ \frac{\cL_\adarm}{(1-\lambda G + \delta)^2} \right] \cdot \lambda\frac{\partial G}{\partial \bB\bA} + \frac{\cL_\adarm}{(1-\lambda G + \delta)^2} \cdot \lambda \frac{\partial^2 G}{\partial \bB\bA^2} \nonumber \\ 
& = \frac{\lambda \frac{\partial G}{\partial \bB\bA} \cdot \frac{\partial \cL_\adarm}{\partial \bB\bA}}{(1-\lambda G + \delta)^2} + \frac{\frac{\partial^2 \cL_\adarm}{\partial \bB\bA^2}}{1-\lambda G + \delta} + \frac{\lambda \cL_\adarm \cdot \frac{\partial^2 G}{\partial \bB\bA^2}}{(1-\lambda G + \delta)^2} \nonumber \\ 
& \quad \quad + \left[ \frac{\frac{\partial \cL_\adarm}{\partial \bB\bA}}{(1-\lambda G + \delta)^2} + \frac{2\lambda \cL_\adarm \cdot \frac{\partial G}{\partial \bB\bA}}{(1-\lambda G + \delta)^3}  \right] \cdot \lambda \frac{\partial G}{\partial \bB\bA} \nonumber \\ 
& = \frac{2\lambda \frac{\partial G}{\partial \bB\bA} \cdot \frac{\partial \cL_\adarm}{\partial \bB\bA}}{(1-\lambda G + \delta)^2} + \frac{2\lambda \cL_\adarm \cdot (\frac{\partial G}{\partial \bB\bA})^2}{(1-\lambda G + \delta)^3} + \frac{\frac{\partial^2 \cL_\adarm}{\partial \bB\bA^2}}{1-\lambda G + \delta} \nonumber \\ 
& \quad \quad + \frac{\lambda \cL_\adarm \cdot \frac{\partial G^2}{\partial \bB\bA^2}}{(1-\lambda G + \delta)^2} \nonumber \\ 
& = \frac{2\lambda \frac{\partial G}{\partial \bB\bA}}{1-\lambda G + \delta} \cdot \left[ \frac{1}{1-\lambda G + \delta} \cdot \frac{\partial \cL_\adarm}{\partial \bB\bA} + \frac{\cL_\adarm}{(1-\lambda G + \delta)^2} \cdot \lambda\frac{\partial G}{\partial \bB\bA} \right] \nonumber \\ 
& \quad \quad + \frac{\frac{\partial^2 \cL_\adarm}{\partial \bB\bA^2}}{1-\lambda G + \delta} + \frac{\lambda \cL_\adarm \cdot \frac{\partial G^2}{\partial \bB\bA^2}}{(1-\lambda G + \delta)^2} \nonumber \\ 
& = \frac{2\nu}{1-\lambda G + \delta} \cdot (\mu^\prime + \nu^\prime) + \frac{\frac{\partial^2 \cL_\adarm}{\partial \bB\bA^2}}{1-\lambda G + \delta} + \frac{\lambda \cL_\adarm \cdot \frac{\partial G^2}{\partial \bB\bA^2}}{(1-\lambda G + \delta)^2} \nonumber \\ 
& = \frac{2\nu}{1-\lambda G + \delta} \cdot \frac{\partial \cL_\splrm}{\partial \bB\bA} + \frac{1}{1-\lambda G + \delta}\cdot \frac{\partial^2 \cL_\adarm}{\partial \bB\bA^2} \nonumber \\
& \quad \quad + \frac{\lambda \cL_\adarm }{(1-\lambda G + \delta)^2}\cdot \frac{\partial G^2}{\partial \bB\bA^2}, \nonumber \\ 
& \mathrm{where}~~\mu = \frac{\partial \cL_\adarm}{\partial \bB\bA}, ~\nu = \lambda \frac{\partial G}{\partial \bB\bA}, \nonumber \\ 
& \quad \quad \quad \ ~\mu^\prime = \frac{1}{1 - \lambda G + \delta} \cdot \mu, ~\nu^\prime = \frac{\cL_{\adarm}}{\left ( 1 - \lambda G + \delta \right )^2} \cdot \nu,
\end{align}
}
leading to:
{\scriptsize
\begin{align}
\|\frac{\partial^2 \mathcal{L}_{\mathrm{SPL}}}{\partial \mathbf{B}\mathbf{A}^2}\| &\le \frac{1}{1-\lambda G + \delta} \cdot \|\frac{\partial^2 \cL_\adarm}{\partial \bB\bA^2}\| + \frac{\lambda \cL_\adarm}{(1-\lambda G + \delta)^2}\cdot \|\frac{\partial G^2}{\partial \bB\bA^2}\| \nonumber \\ 
&+ \frac{2\|\nu\|}{1-\lambda G + \delta}\cdot\|\frac{\partial \mathcal{L}_{\mathrm{SPL}}}{\partial \mathbf{B}\mathbf{A}}\|, \nonumber
\end{align}
}
resulting in:
{\scriptsize
\begin{align}
& \|\frac{\partial^2 \mathcal{L}_{\mathrm{SPL}}}{\partial \mathbf{B}\mathbf{A}^2}\| \le L^\prime_{0} + L^\prime_{1}\|\frac{\partial \mathcal{L}_{\mathrm{SPL}}}{\partial \mathbf{B}\mathbf{A}}\|, \nonumber \\
& \mathrm{where}~~L^\prime_{0} = \frac{1}{1-\lambda G + \delta} \cdot \|\frac{\partial^2 \cL_\adarm}{\partial \bB\bA^2}\| + \frac{\lambda \cL_\adarm}{(1-\lambda G + \delta)^2}\cdot \|\frac{\partial G^2}{\partial \bB\bA^2}\|, \nonumber \\
& L^\prime_{1} = \frac{2\|\nu\|}{1-\lambda G + \delta}. \nonumber
\end{align}
}

\section{Detailed Calculation of the Gradient Norm and the Hessian Norm}
\label{app:CalGradHess}
In this section, we detail the calculation of the gradient norm and the Hessian norm achieved by the training loss for both MP-LoRA and SMP-LoRA.

For MP-LoRA, the training loss gradient w.r.t. parameters $\bB$ and $\bA$ are calculated as follows:
\begin{equation}
\label{eq:PLoRA_Grad}
 \frac{\partial \cL_\plrm}{\partial \bB} =  (\mu + \nu)\bA^\top ,  
~\frac{\partial \cL_\plrm}{\partial \bA} = \bB^\top (\mu + \nu) ,   
\end{equation}
where $\mu$ and $\nu$ are calculated in Equation~\ref{eq:SPLoRA_SecDer}. The gradient norm for the training loss during MP-LoRA, as shown in Figure~\ref{fig:fig1b_GradientHessian}, is calculated as the square root of the sum of the squared gradients across all parameters in $\bB$ and $\bA$, which is equivalent to computing the overall L2 norm for all gradients, i.e., $\sqrt{\| \frac{\partial \cL_\plrm}{\partial \bB} \|_2^2 + \| \frac{\partial \cL_\plrm}{\partial \bA}\|_2^2}$.

For SMP-LoRA, the training loss gradient is calculated as follows:
\begin{align}
\label{eq:SPLoRA_Grad}
\frac{\partial \cL_\splrm}{\partial \bB} = (\mu^\prime + \nu^\prime)\bA^\top , 
\frac{\partial \cL_\splrm}{\partial \bA} = \bB^\top (\mu^\prime + \nu^\prime), 
\end{align}
where $\mu^\prime$ and $\nu^\prime$ are calculated in Equation~\ref{eq:SPLoRA_SecDer}. The gradient norm for the training loss during SMP-LoRA, as shown in Figure~\ref{fig:fig1b_GradientHessian}, is calculated as $\sqrt{\| \frac{\partial \cL_\splrm}{\partial \bB} \|_2^2 + \| \frac{\partial \cL_\splrm}{\partial \bA}\|_2^2}$

For the norm of the Hessian matrix (Hessian norm), calculating the actual norm is impractical due to the significant computational cost involved.
Consequently, we use power iteration to iteratively approximate the largest eigenvalue of the Hessian matrix, which is a reliable indicator of the matrix norm, to estimate the Hessian norm.
During the adaptation process via MP-LoRA and SMP-LoRA, we estimate the Hessian norm every 100 steps, with the power iteration configured to run for 50 iterations and a tolerance of 1e-6.

\begin{table}[t]
  \caption{This table displays the maximum, minimum, mean, and standard deviation of the dynamically changed rescaling factors introduced by SMP-LoRA during the adaptation process on the Pokemon dataset.}
  \label{tab:tab3_ResFactor}
  \centering
  \begin{tabular}{ccccc}
    \toprule
    Rescaling Factors & Max & Min & Mean & Std Dev \\
    \midrule
    $ \frac{1}{1 - \lambda G + \delta} $    & 0.999 & 0.018 & 0.946 & 0.087 \\
    \midrule
    $\frac{\cL_{\adarm}}{\left ( 1 - \lambda G + \delta \right )^2}$ & 0.037 & 2.9e-6 & 0.002 & 0.003 \\
    \bottomrule
  \end{tabular}
\end{table}

\section{Further Empirical Analysis of Gradients}
\label{app:GradScale}

In this section, we provide a detailed empirical analysis to compare the gradient of LoRA, MP-LoRA and SMP-LoRA. We empirically show that compared to LoRA, MP-LoRA exhibits significant fluctuations in gradient scale, especially obtained by the MI gain. In contrast, SMP-LoRA introduce specific factors that implicitly rescale the gradient, effectively stabilizing the gradient and controlling the gradient scale.

We firstly tracked and demonstrated the gradient scales of the training loss for LoRA and MP-LoRA at each training iteration in Figure~\ref{fig:fig2a_GradScale}. The gradient scale is calculated as the sum of the L2 norm of the gradient across all parameters in $\bB$ and $\bA$, i.e., $\| \frac{\partial \cL_\plrm}{\partial \bB} \|_2 + \| \frac{\partial \cL_\plrm}{\partial \bA}\|_2$. Figure~\ref{fig:fig2a_GradScale} shows that the standard deviation of gradient scales obtained by MP-LoRA (1.95) is much higher than that obtained by LoRA (0.42). This significant difference in standard deviation suggests exploring the specific components contributing to the fluctuations.

Consequently, we analyze the gradient scales achieved by the adaptation loss and the MI gain during MP-LoRA, which are calculated as $\| \mu \bA^\top \|_2 + \| \bB^\top \mu\|_2$ and $\| \nu \bA^\top \|_2 + \| \bB^\top \nu \|_2 $ respectively, in Figure~\ref{fig:fig2b_PLoRAGradScale}. We observe that the standard deviation of the MI gain's gradient scale reaches 3.27, which is significantly higher than that achieved by the adaptation loss (1.43). Therefore, it suggests that the introduced MI gain that aims to protect membership privacy could be the primary cause of significant fluctuations in the gradient scales.

In this context, the gradient scale obtained by the adaptation loss and the MI gain during SMP-LoRA can be calculated as $ \frac{1}{1 - \lambda G + \delta} (\| \mu \bA^\top \|_2 + \| \bB^\top \mu\|_2)$ and $ \frac{\cL_{\adarm}}{\left ( 1 - \lambda G + \delta \right )^2} (\| \nu \bA^\top \|_2$ + $\| \bB^\top \nu \|_2) $ respectively. Therefore, compared to MP-LoRA, the gradient scales of the adaptation loss and the MI gain during SMP-LoRA are implicitly rescaled by the factors $ \frac{1}{1 - \lambda G + \delta} $ and $\frac{\cL_{\adarm}}{\left ( 1 - \lambda G + \delta \right )^2}$, respectively. Observed from Figure~\ref{fig:fig2c_SPLoRAGradScale}, SMP-LoRA obtains a much lower standard deviation of the gradient scale compared to MP-LoRA. Notably, the rescaling factors are not constant but dynamically change with the adaptation process, as illustrated in Table~\ref{tab:tab3_ResFactor}.

\begin{table}[t]
  \caption{This table displays the sizes of the four subsets across all three datasets.}
  \small
  \label{tab:tab4_Dataset}
  \centering
  \begin{tabular}{lccccr}
    \toprule
    Dataset & $\left | \cD_\auxrm^\mrm \right |$ & $\left | \cD_\term^\mrm \right |$ & $\left | \cD_\auxrm^\nmrm \right |$ & $\left | \cD_\term^\nmrm \right |$ \\
    \midrule
    Pokemon    & 200 & 200 & 200 & 233 \\
    CelebA$\_$Small & 50 & 50 & 50 & 50 \\
    CelebA$\_$Large & 200 & 200 & 200 & 200 \\
    CelebA$\_$Gender & 40 & 40 & 40 & 40 \\
    CelebA$\_$Varying & 63 & 63 & 63 & 63 \\
    CelebA$\_$Large$\_$5X & 1000 & 1000 & 1000 & 1000 \\
    AFHQ$\_$Large$\_$5X & 750 & 750 & 750 & 750 \\
    MS-COCO$\_$Large$\_$5X & 1000 & 1000 & 1000 & 1000 \\
    \bottomrule
  \end{tabular}
\end{table}

\section{Experimental Setup}
\label{app:ExpSetup}
\paragraph{Experimental environments.} 
We conducted all experiments on Python 3.10.6 (PyTorch 2.0.0 + cu118) with NVIDIA A100-SXM4 GPUs (CUDA 12.1).
The GPU memory usage and running time for the adaptation experiments via LoRA and SMP-LoRA depend on the dataset size. 
Specifically, for the Pokemon dataset with 400 epochs running, LoRA consumes approximately 8,000 MiB of GPU memory and requires about 16 hours, while SMP-LoRA uses roughly 10,000 MiB and takes around 20 hours.

\paragraph{Dataset.} 
The CelebA$\_$Small dataset consists of 200 images, from 25 randomly selected individuals, each providing 8 images.
Similarly, the CelebA$\_$Large dataset contains 800 images from 100 randomly selected individuals.
The CelebA$\_$Gender dataset is designed to simulate real-world adaptation tasks with a gender imbalance. It includes 160 images with a gender ratio of $7:3$, comprising 20 randomly selected individuals, each providing 8 images.
The CelebA$\_$Varying dataset simulates varying image contributions per individual, created by sampling uniformly from the range $[1, 4]$, resulting in a total of 252 images from 100 randomly selected individuals.
The CelebA$\_$Large$\_$5X dataset is expanded to five times the size of CelebA$\_$Large by using a fixed random seed for sampling, with the AFHQ$\_$Large$\_$5X and MS-COCO$\_$Large$\_$5X datasets created using the same approach.
The Pokemon dataset contains text for each image which is generated by the pre-trained BLIP~\citep{Li2022blip} model.
We also utilized the pre-trained BLIP model for the CelebA$\_$Small, CelebA$\_$Large, CelebA$\_$Gender, CelebA$\_$Varying, CelebA$\_$Large$\_$5X, AFHQ$\_$Large$\_$5X, and MS-COCO$\_$Large$\_$5X datasets to generate corresponding text for each image.

For our experiments, each dataset was divided into four subsets: $\cD_\auxrm^\mrm$, $\cD_\term^\mrm$, $\cD_\auxrm^\nmrm$, and $\cD_\term^\nmrm$, with the detailed information provided in Table~\ref{tab:tab4_Dataset}.
The dataset $\cD_\trrm = \cD_\auxrm^\mrm \cup \cD_\term^\mrm$ was used for adapting LDMs via LoRA.
For MP-LoRA and SMP-LoRA, an additional dataset $\cD_\auxrm= \cD_\auxrm^\mrm \cup \cD_\auxrm^\nmrm$ was utilized.
To evaluate the effectiveness of adapted LDMs against MI attacks, we employed the dataset $\cD_\term= \cD_\term^\mrm \cup \cD_\term^\nmrm$.

\begin{table}[t]
  \caption{Hyperparameter settings for the LoRA-adapted LDM. LoRA $r$ is the rank used in the decomposition of the frozen weight matrix, as detailed in Section~\ref{sec2.2_LoRA}. LoRA $\alpha$ is a scaling constant applied to the output of the LoRA module $\mathbf{B}\mathbf{A}$.} LR schedule refers to the learning rate schedule during training.
  \label{tab:tab5_Hyperparameter}
  \centering
  \begin{tabular}{llll}
    \toprule
    Model &  \multicolumn{3}{l}{Stable Diffusion v1.5} \\
    \midrule
    LoRA $r$ & 64 & LoRA $\alpha$ & 32 \\
    Diffusion steps & 1000 & Noise schedule & linear \\
    Resolution & 512 & Batch size & 1 \\
    Learning rate $\eta_1$ & \textnormal{1e-5} & Learning rate $\eta_2$ & \textnormal{1e-4} \\
    LR schedule & Constant  & Training epochs & 400 \\
\bottomrule
  \end{tabular}
\end{table}

\paragraph{Model hyperparameters.} 
First, for adapting LDMs, we used the official pre-trained Stable Diffusion v1.5~\citep{Compvis2024stable}.
The model hyperparameters, such as the rank $r$ for LoRA and diffusion steps, are detailed in Table~\ref{tab:tab5_Hyperparameter}.
Second, the proxy attack model $h_\omega$ is a 3-layer MLP with layer sizes [512, 256, 2], where the final layer is connected to a softmax function to output probability.
Third, to evaluate the effectiveness of adapted LDMs in defending against MI attacks, we trained a new attack model $h^\prime$ on the auxiliary dataset $\cD_\auxrm$ for 100 epochs.
This model $h^\prime$, structurally based on $h_\omega$, was then used to conduct MI attacks on the dataset $\cD_\term$.
Both $h_\omega$ and $h^\prime$ were optimized with the Adam optimizer with a learning rate of 1e-5.
Furthermore, to prevent the attack model from being biased towards one side, we ensured that each training batch for both $h_\omega$ and $h^\prime$ contained an equal number of member and non-member data points.

\begin{table*}[!t]
  \caption[The effect of coefficient $\lambda$ for SMP-LoRA.]{The effect of coefficient $\lambda$ for SMP-LoRA on the Pokemon, CelebA\_Small, and CelebA\_Large dataset.}
  \label{tab:tab6_aba_lambda}
  \centering
  \begin{tabular}{llcccc}
    \toprule
    \textbf{Dataset} & $\lambda$ & FID \textbf{$\downarrow$} & KID \textbf{$\downarrow$} & ASR (\%) \textbf{$\downarrow$} & $\frac{\left | \mathrm{AUC} - 0.5 \right |}{0.5}$ \textbf{$\downarrow$} \\
    \cmidrule(lr){1-6}
    \multirow{5}{*}{Pokemon} & 1.00 & 1.633 & 0.0576 & \textbf{46.0} & 0.16  \\
    &0.50 & 1.769 & 0.0694 & 46.2 & \textbf{0.08} \\
    &0.10 & 0.361 & 0.0063 & 47.1 & \textbf{0.08} \\
    &0.05 & 0.274 & 0.0057 & 50.3 & 0.16 \\
    &0.01 & \textbf{0.214} & \textbf{0.0048} & 61.2 & 0.20 \\
    \cmidrule(lr){1-6}
    \multirow{5}{*}{CelebA\_Small} & 1.00 & N/A & 0.1867 & 54.0 & \textbf{0.08}  \\
     & 0.50 & N/A & 0.1004 & \textbf{51.0} & 0.24 \\
     & 0.10 & N/A & 0.0592 & 55.0 & \textbf{0.08} \\
     & 0.05 & N/A & \textbf{0.0235} & 54.0 & 0.18 \\
     & 0.01 & N/A & 0.0545 & 78.0 & 0.70 \\
    \cmidrule(lr){1-6}
    \multirow{5}{*}{CelebA\_Large} & 1.00 & 1.137 & 0.1116 & 52.0 & 0.04 \\
     & 0.50 & 1.489 & 0.1559 & 52.8 & \textbf{0.02} \\
     & 0.10 & 0.835 & 0.0767 &\textbf{50.0} & 0.16 \\
     & 0.05 & \textbf{0.556} & \textbf{0.0512} & \textbf{50.0} & 0.08 \\
     & 0.01 & 0.590 & 0.0582 & 57.3 & 0.16 \\
\bottomrule
  \end{tabular}
\end{table*}

\begin{figure*}[!t]
  \centering
  \begin{subfigure}[t]{0.32\textwidth}
    \centering
    \includegraphics[width=0.9\textwidth]{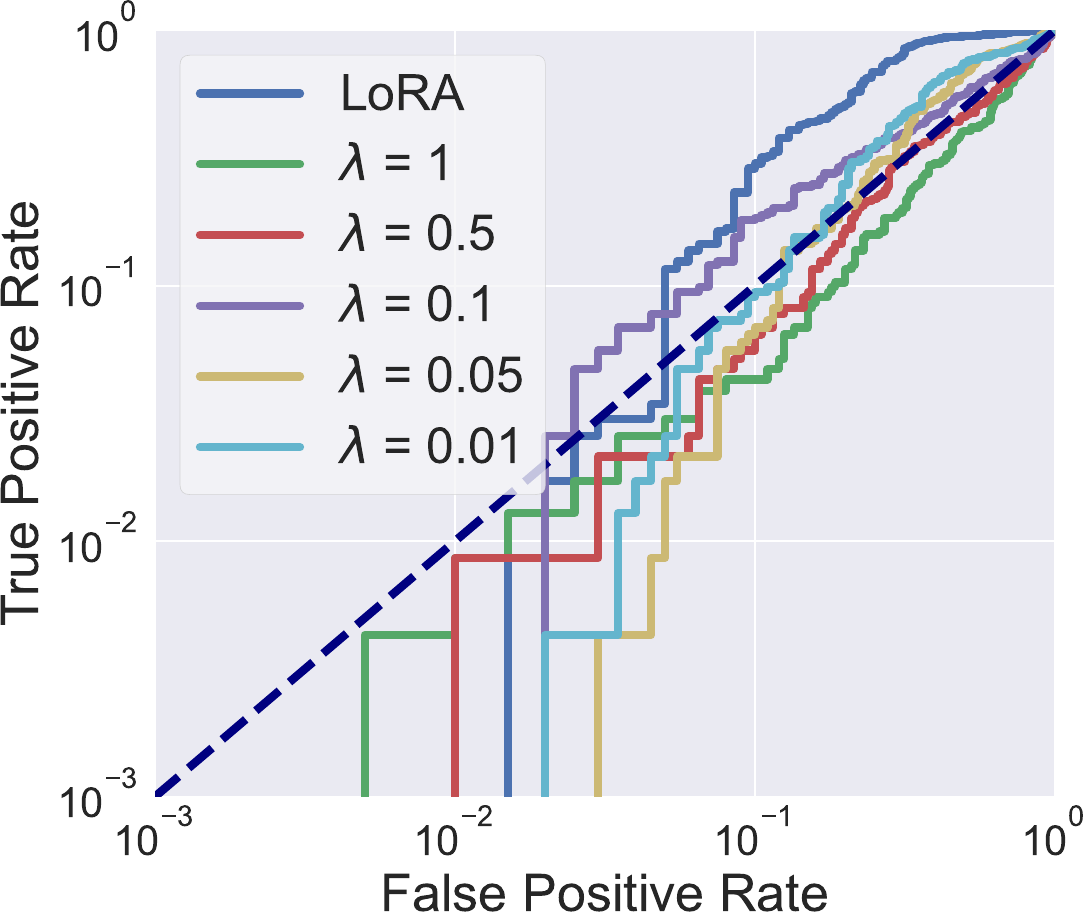}
    \subcaption{Pokemon dataset.}
    \label{fig:fig3a_ROCCL}
  \end{subfigure}
  \hfill
  \begin{subfigure}[t]{0.32\textwidth}
    \centering
    \includegraphics[width=0.9\textwidth]{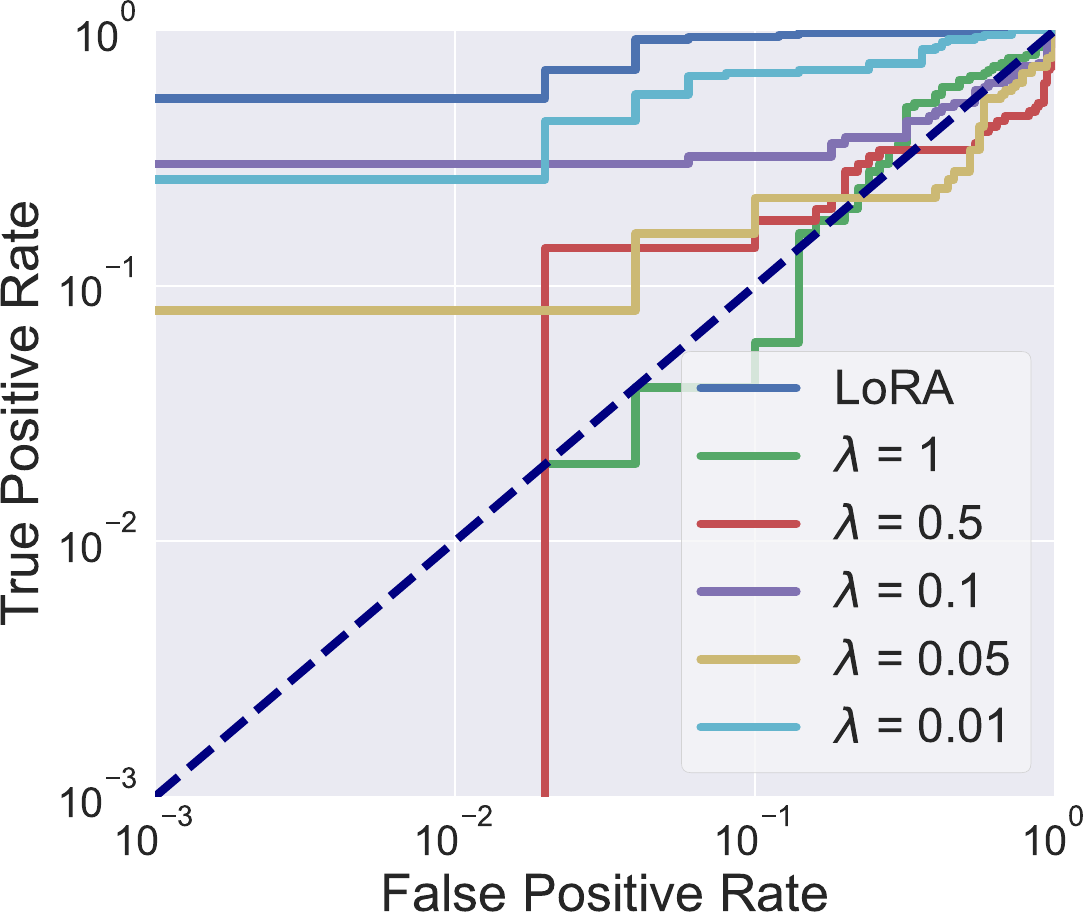}
    \subcaption{CelebA$\_$Small dataset.}
    \label{fig:fig3b_ROCCL}
  \end{subfigure}
  \hfill
  \begin{subfigure}[t]{0.32\textwidth}
    \centering
    \includegraphics[width=0.9\textwidth]{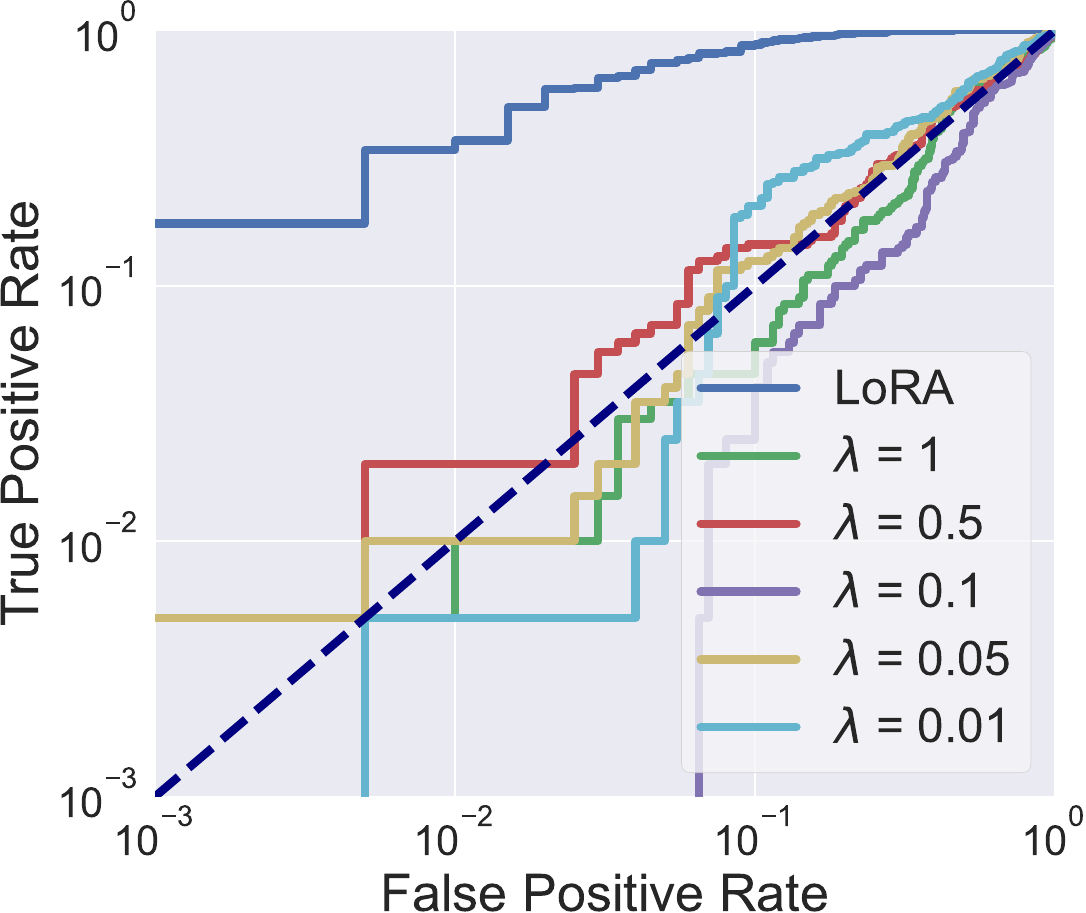}
    \subcaption{CelebA$\_$Large dataset.}
    \label{fig:fig3c_ROCCL}
  \end{subfigure}
  \hfill
  \caption{ROC curves for SMP-LoRA with different $\lambda$ on the Pokemon, CelebA$\_$Small, and CelebA$\_$Large datasets.}
  \label{fig:fig3}
\end{figure*}

\paragraph{Evaluation metrics.} 
We employed Attack Success Rate (ASR)~\citep{Choquette2021label}, Area Under the ROC Curve (AUC), and True Positive Rate (TPR) at a fixed False Positive Rate (FPR) of $5\%$ to evaluate the effectiveness of LoRA, MP-LoRA, and SMP-LoRA in defending against MI attacks
Consistent with many prior studies~\citep{Chen2022relaxloss,Ye2022enhanced,Duan2023diffusion,Dubinski2024towards}, we default that the TPR measures the capability of the attack model to identify samples as members correctly.
ASR is calculated by dividing the number of correctly identified members and non-members by the total number of samples.
Therefore, the checkpoint of the $h^\prime$ with the highest ASR was identified as the point with maximum MI gain, and the corresponding AUC and TPR values were calculated.
Lower values for ASR and TPR indicate a more effective defense against MI attacks.

For AUC, it is crucial to recognize that MI attacks not only attempt to determine members but also non-members.
Thus, from the defender's perspective, a lower AUC does not necessarily indicate stronger protection.
An AUC value closer to 0.5 (random level), i.e., a smaller value of $\frac{\left | \mathrm{AUC} - 0.5 \right |}{0.5}$, indicates a stronger defensive capability.

For assessing the image generation capability of the adapted LDMs, we utilized the Fr\'{e}chet Inception Distance (FID)~\citep{Heusel2017gans} score and the Kernel Inception Distance (KID)~\citep{Binkowski2018demystifying} score. 
Specifically, the FID score is calculated based on 768-dimensional feature vectors.
We need the number of data points equal to or greater than the dimensions of the feature vector to ensure a full-rank covariance matrix for a correct FID score calculation.
Therefore, for the CelebA$\_$Small, CelebA$\_$Gender, and CelebA$\_$Varying datasets, which contain less than 768 data points, we only calculated the KID scores.
Lower values of FID and KID scores indicate better image quality and greater similarity between the generated and training images.

\section{Detailed Ablation Study}
\label{app:detailed_as}
This section provides a detailed ablation study, including all experimental results and thorough analyses.

\paragraph{Coefficient $\lambda$.}
Table~\ref{tab:tab6_aba_lambda} present the performance of SMP-LoRA with different coefficient $\lambda \in \{1.00, 0.50, 0.10, 0.05, 0.01\}$ across the Pokemon, CelebA$\_$Small, and CelebA$\_$Large datasets.
As $\lambda$ decreases from 1.00 to 0.01, the FID and KID scores gradually decrease, while ASR increases and AUC deviates further from 0.5.
This suggests that a lower $\lambda$ shifts the focus more towards minimizing adaptation loss rather than protecting membership privacy.
When $\lambda = 0.01$, the ASR exceeds 50\%, and the AUC deviates significantly from 0.5, indicating insufficient protection of membership privacy by SMP-LoRA.
Therefore, at $\lambda = 0.05$, we consider SMP-LoRA to exhibit optimal performance, effectively preserving membership privacy with minimal cost to image generation capability.
For the experiments in Section~\ref{sec4.1_Performance} and subsequent ablation studies, $\lambda$ is set at 0.05 for both MP-LoRA and SMP-LoRA.

Figure~\ref{fig:fig3} shows the ROC curves for SMP-LoRA with these $\lambda$ values across the Pokemon, CelebA$\_$Small, and CelebA$\_$Large datasets.
We can observe that, in most cases, when $\lambda \geq 0.05$, SMP-LoRA maintains a low True Positive Rate (TPR) at 0.1\%, 1\%, and 10\% False Positive Rate (FPR) across all three datasets.
This demonstrates the effectiveness of SMP-LoRA in defending MI attacks, whether under strict FPR constraints or more lenient error tolerance conditions.

\begin{table}[!t]
  \caption{The effect of learning rate $\eta_{2}$ for SMP-LoRA on the Pokemon dataset.}
  \label{tab:tab7_LR}
  \centering
  \begin{tabular}{lccc}
    \toprule
    $\eta_{2}$ & FID \textbf{$\downarrow$} & ASR (\%) \textbf{$\downarrow$} & $\frac{\left | \mathrm{AUC} - 0.5 \right |}{0.5}$ \textbf{$\downarrow$} \\
    \midrule
    \textnormal{1e-4} & \textbf{0.274} & 50.3 & 0.16 \\
    \textnormal{1e-5} & 0.436 & 46.2 & 0.10 \\
    \textnormal{1e-6} & 0.802 & \textbf{45.9} & \textbf{0.04} \\
    \bottomrule
  \end{tabular}
\end{table}

\paragraph{Learning rate $\eta_{2}$.}
Table~\ref{tab:tab7_LR} displays the performance of SMP-LoRA with different learning rate $\eta_{2} \in \{1e-4, 1e-5, 1e-6\}$ on the Pokemon dataset.
We observe an increase in the FID scores as the learning rate $\eta_{2}$ decreases.
This phenomenon might be due to the lower learning rate, which results in the model underfitting after 400 training epochs.
Notably, across all tested learning rates, SMP-LoRA consistently preserves membership privacy, effectively defending against MI attacks. 

\paragraph{LoRA's rank $r$.}
Table~\ref{tab:tab8_Rank} shows the performance of SMP-LoRA with different rank $r \in \{128, 64, 32, 16, 8\}$ on the Pokemon dataset.
We observe that the LoRA's rank $r$ does not significantly affect the performance of SMP-LoRA.

\paragraph{Extending to the full fine-tuning and DreamBooth Methods.}
Table~\ref{tab:tab9_finetuning} presents the performance of SMP-LoRA and its extension to the full fine-tuning and DreamBooth~\citep{Ruiz2023dreambooth} methods on the Pokemon dataset.
Both the full fine-tuning and DreamBooth methods are sensitive to the learning rate.
Accordingly, the learning rate $\eta_{2}$ for both methods was set at 5e-6, while other hyperparameters remained consistent with our experimental setup in Appendix~\ref{app:ExpSetup}.
In Table~\ref{tab:tab9_finetuning}, we can observe that our membership-privacy-preserving method effectively protects membership privacy when applied to both methods.
This underscores the potential applicability of our membership-privacy-preserving method across different adaptation methods.

\begin{table}[!t]
  \caption{The effect of LoRA's rank $r$ for SMP-LoRA on the Pokemon dataset.}
  \label{tab:tab8_Rank}
  \centering
  \begin{tabular}{lccc}
    \toprule
    $r$ & FID \textbf{$\downarrow$} & ASR (\%) \textbf{$\downarrow$} & $\frac{\left | \mathrm{AUC} - 0.5 \right |}{0.5}$ \textbf{$\downarrow$}  \\
    \midrule
    128 & 0.337 & 56.1 & \textbf{0.02} \\
    64 & \textbf{0.274} & 50.3 & 0.16 \\
    32 & 0.512 & 54.0 & 0.14 \\
    16 & 0.541 & 53.1 & 0.16 \\
    8 & 0.661 & \textbf{48.9} & 0.08 \\
    \bottomrule
  \end{tabular}
\end{table}

\begin{table*}[!t]
  \caption{Performance of SMP-LoRA and its extension to the full fine-tuning and DreamBooth~\citep{Ruiz2023dreambooth} methods on the Pokemon dataset.}
  \label{tab:tab9_finetuning}
  \centering
  \begin{tabular}{lccc}
    \toprule
    Method & FID \textbf{$\downarrow$} & ASR (\%) \textbf{$\downarrow$} & $\frac{\left | \mathrm{AUC} - 0.5 \right |}{0.5}$ \textbf{$\downarrow$} \\
    \midrule
    LoRA & \textbf{0.20$\pm$0.04} & 82.27$\pm$4.38 & 0.73$\pm$0.09 \\
    SMP-LoRA & 0.32$\pm$0.07 & \textbf{51.97$\pm$1.20} & \textbf{0.14$\pm$0.02} \\
    \midrule
    Full Fine-tuning & \textbf{0.176} & 80.4 & 0.66 \\
    SMP Full Fine-tuning& 1.05 & \textbf{54.5} & \textbf{0.26} \\
    \midrule
    DreamBooth & \textbf{0.260} & 80.6 & 0.70 \\
    SMP-DreamBooth & 0.748 & \textbf{56.4} & \textbf{0.10} \\
    \bottomrule
  \end{tabular}
\end{table*}

\begin{table*}[t]
  \caption{Performance of SMP-LoRA and MP-LoRA enhanced with gradient clipping and normalization techniques on the Pokemon dataset.}
  \label{tab:tab10_ClipNorm}
  \centering
  \begin{tabular}{lccc}
    \toprule
    Method & FID \textbf{$\downarrow$} & ASR (\%) \textbf{$\downarrow$} & $\frac{\left | \mathrm{AUC} - 0.5 \right |}{0.5}$ \textbf{$\downarrow$} \\
    \midrule
    MP-LoRA + Gradient Clipping & 3.513 & \textbf{46.2} & 0.14 \\
    MP-LoRA + Gradient Normalization & 2.390 & 52.4 & \textbf{0.06} \\
    SMP-LoRA & \textbf{0.274} & 50.3 & 0.16 \\
    \bottomrule
  \end{tabular}
\end{table*}

\begin{table*}[!t]
  \caption{Performance of LoRA and SMP-LoRA under the black-box~\citep{Wu2022membership} and the white-box gradient-based~\citep{Pang2023white} MI attacks on the Pokemon dataset.}
  \label{tab:tab11_BlackWhite}
  \centering
  \begin{tabular}{llcc}
    \toprule
    Attack & Method & ASR (\%) \textbf{$\downarrow$} & $\frac{\left | \mathrm{AUC} - 0.5 \right |}{0.5}$ \textbf{$\downarrow$} \\
    \midrule
    \multirow{2}{*}{Black-box} & LoRA & 72.4 & 0.34 \\
    & SMP-LoRA & \textbf{55.4} & \textbf{0.08} \\
    \midrule
    \multirow{2}{*}{White-box Gradient-based} & LoRA & 92.4 & 0.84 \\
    & SMP-LoRA & \textbf{63.5} & \textbf{0.28} \\
    \bottomrule
  \end{tabular}
\end{table*}

\begin{figure*}[!h]
  \centering
  \includegraphics[width=0.8\textwidth,keepaspectratio]{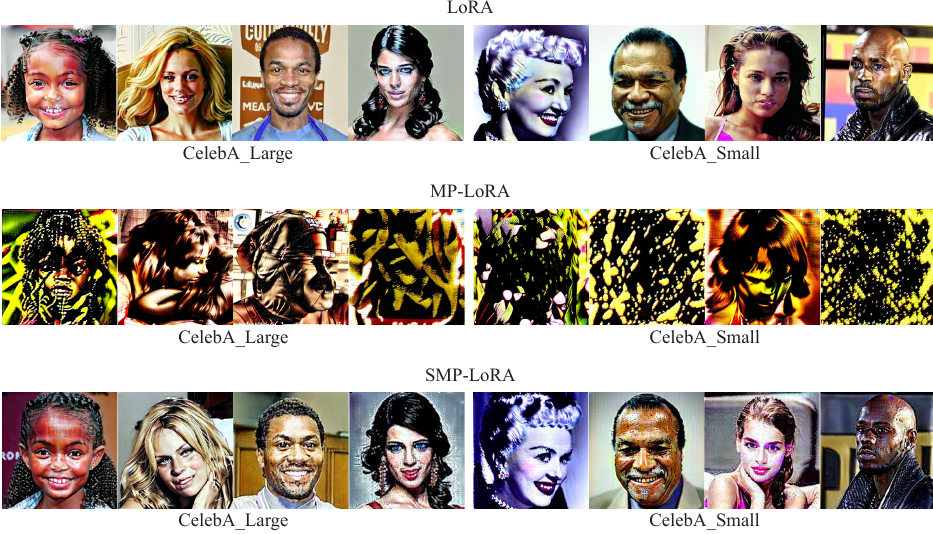} 
  \caption{Generated results on the CelebA$\_$Small and CelebA$\_$Large datasets. Each column of three images is generated using the same text prompt.}
  \label{fig:fig4_GeneratedResults} 
\end{figure*}

\paragraph{Comparing with gradient clipping and normalization techniques.}
In Table~\ref{tab:tab10_ClipNorm}, we report the performance of SMP-LoRA and MP-LoRA enhanced with gradient clipping and normalization techniques on the Pokemon dataset.
Notably, gradient clipping is a standard technique in training Latent Diffusion Models and has been employed in our experiments.
In Table~\ref{tab:tab10_ClipNorm}, the higher FID scores for gradient clipping (3.513) and gradient normalization (2.390) reflect poor image quality, indicating the failure in adaptation.
These results demonstrate that such traditional techniques for stabilizing the gradient, as gradient clipping and normalization, cannot effectively address the unstable optimization issue in MP-LoRA.

\paragraph{Defending against MI attacks in different settings.}
Table~\ref{tab:tab11_BlackWhite} displays the attack performance on LoRA and SMP-LoRA using the black-box MI attack~\citep{Wu2022membership} and the white-box gradient-based MI attack~\citep{Pang2023white}, which is the currently the most potent MI attack targeting DMs.
Compared to LoRA, SMP-LoRA, specifically designed to defend against white-box loss-based MI attacks, exhibits lower ASR and better AUC, indicating that it can still preserve membership privacy to a certain extent when facing MI attacks in different settings.
Until now, we have evaluated SMP-LoRA against the white-box loss-based MI attack and the currently strongest MI attack, the white-box gradient-based MI attack, rendering further comparisons with weaker attacks unnecessary, such as gray-box MI attacks~\citep{Duan2023diffusion,Kong2023efficient,Fu2023probabilistic}.
The black-box MI attack was implemented from the semantic-based Attack II-S proposed by~\citet{Wu2022membership}, and the white-box gradient-based MI attack was replicated from the GSA$_{1}$ proposed by~\citet{Pang2023white}.
Implementation details for both MI attacks are available in Appendix~\ref{app:BlackboxImple}.

\section{Implementation Details}
\label{app:BlackboxImple}
In this section, we detail the implementation of the black-box and white-box gradient-based MI attacks used in Section~\ref{sec4.2_AblationStudy}.
The attack performance of both attacks on LoRA-adapted and SMP-LoRA-adapted LDMs is presented in Table~\ref{tab:tab11_BlackWhite}.

For the black-box MI attack, we utilize the semantic-based Attack II-S proposed by \citet{Wu2022membership}. This attack leverages the pre-trained BLIP model to extract embeddings for both an image and its corresponding text-generated image, and then conduct MI attacks based on the L2 distance between these two embeddings. Based on \citet{Wu2022membership}'s experiment setup, we instantiate the attack model as a 3-layer MLP with cross-entropy loss, optimized with Adam at a learning rate of 1e-4 over 200 training epochs.

For the white-box gradient-based MI attack, we replicate the GSA$_{1}$ proposed by \citet{Pang2023white}. This attack involves uniformly sampling ten steps from the total diffusion steps, calculating the loss at each step, averaging these losses, and then performing backpropagation to obtain gradients. These gradients are then used to train an XGBoost model to infer membership and non-membership.

\section{Visualization of Generated Results}
\label{app:Visul_mpsmplora}
We provide more generated results of LoRA-adapted, MP-LoRA-adapted, and SMP-LoRA-adapted LDMs on the CelebA$\_$Small and CelebA$\_$Large datasets in Figure~\ref{fig:fig4_GeneratedResults}.

\section{Discussion}
\subsection{MP-LoRA vs. SMP-LoRA}
In Table~\ref{tab:tab1_Performance}, it is evident that MP-LoRA significantly enhances membership privacy but compromises image generation capabilities, while SMP-LoRA effectively preserves membership privacy without impairing image quality. 
Moreover, in terms of ASR and TPR@$5\%$FPR metrics, SMP-LoRA sometimes outperforms MP-LoRA in preserving membership privacy.
Based on these findings, we cannot assert that MP-LoRA provides better membership privacy protection compared to SMP-LoRA, as MP-LoRA is not fully optimized, evidenced by its complete loss of image generation capability. 
Therefore, SMP-LoRA is not just a balance within the utility-privacy Pareto Optimality.

\section{Limitations}
Our SMP-LoRA method is primarily designed to defend against white-box loss-based MI attacks, yet there are other aspects of privacy vulnerability, such as Model Inversion Attack~\citep{Carlini2023extracting}.
Besides, our method is currently applied solely to LoRA (tested on full fine-tuning and DreamBooth), and its extension to other adaptation methods, such as Textual Inversion~\citep{Gal2022image} and Hypernetwork~\citep{Hypernetwork}, remains unexplored.
Additionally, our selection of the coefficient value $\lambda$ relies on empirical validations, lacking a principled way.

\end{document}